\documentclass[journal]{IEEEtran}
\usepackage{amsmath,amsfonts}
\usepackage{algorithmic}
\usepackage{algorithm}
\usepackage{array}
\usepackage[caption=false,font=normalsize,labelfont=sf,textfont=sf]{subfig}
\usepackage{textcomp}
\usepackage{stfloats}
\usepackage{url}
\usepackage{verbatim}
\usepackage{graphicx}
\usepackage{cite}
\usepackage{colortbl}  %
\usepackage{xcolor}
\usepackage{array}   %
\usepackage{graphicx}
\usepackage{enumitem}
\usepackage{bbding} 
\usepackage{threeparttable}
\usepackage{multicol,multirow}
\usepackage{booktabs} %
\usepackage{amsthm}
\usepackage{makecell}
\newtheorem{theorem}{Theorem}
\newtheorem{definition}{Definition}

\usepackage{hyperref}       
\usepackage{titletoc}

\begin{document}

\title{Toward Robust and Harmonious Adaptation for Cross-modal Retrieval}

\author{Haobin Li, Mouxing Yang, Xi Peng
\IEEEcompsocitemizethanks{
\IEEEcompsocthanksitem 
H. Li, M. Yang, and X. Peng are with the College of Computer Science, Sichuan University, China. X. Peng is also with the
National Key Laboratory of Fundamental Algorithms and Models for Engineering Numerical Simulation, Sichuan University, China. E-mail: \{haobinli.gm, yangmouxing, pengx.gm\}@gmail.com.
}
\thanks{Corresponding Authors: M. Yang and X. Peng.}
}

\maketitle

\begin{abstract} 
Recently, the general-to-customized paradigm has emerged as the dominant approach for Cross-Modal Retrieval (CMR), which reconciles the distribution shift problem between the source domain and the target domain.
However, existing general-to-customized CMR methods typically assume that the entire target-domain data is available, which is easily violated in real-world scenarios and thus inevitably suffer from the query shift (QS) problem.
Specifically, query shift embraces the following two characteristics and thus poses new challenges to CMR.
i) \textit{Online Shift}: real-world queries always arrive in an online manner, rendering it impractical to access the entire query set beforehand for customization approaches;
ii) \textit{Diverse Shift}: even with domain customization, the CMR models struggle to satisfy queries from diverse users or scenarios, leaving an urgent need to accommodate diverse queries.
In this paper, we observe that QS would not only undermine the well-structured common space inherited from the source model, but also steer the model toward forgetting the indispensable general knowledge for CMR.
Inspired by the observations, we propose a novel method for achieving online and harmonious adaptation against QS, dubbed Robust adaptation with quEry ShifT (REST).
To deal with online shift, REST first refines the retrieval results to formulate the query predictions and accordingly designs a QS-robust objective function on these predictions to preserve the well-established common space in an online manner.
As for tackling the more challenging diverse shift, REST employs a gradient decoupling module to dexterously manipulate the gradients during the adaptation process, thus preventing the CMR model from forgetting the general knowledge.
Extensive experiments on 20 benchmarks across three CMR tasks verify the effectiveness of our method against QS. 
\end{abstract}


\begin{figure*}[t]
\centering
\includegraphics[width=0.98\linewidth]{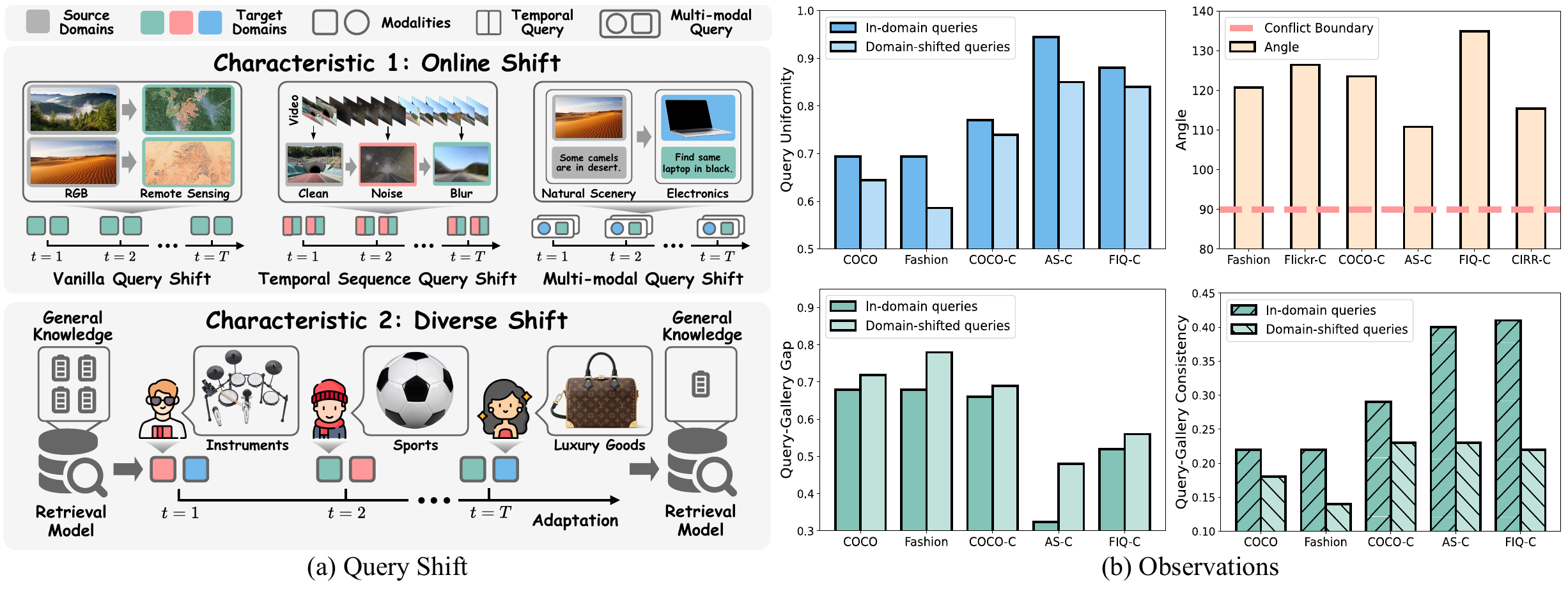}
\caption{
(a) \textbf{Query Shift.} 
The general-to-customized CMR paradigm would suffer from the query shift problem, which exhibits online shift and diverse shift characteristics.
On the one hand, in real-world applications, inquirers would incrementally offer queries, thus forming the online query stream.
Among various CMR tasks, queries exhibit not only uni-distribution shift but also more complex multi-distribution shift, both of which would invalidate the general-to-customized paradigm.
On the other hand, the submitted queries may originate from highly-personalized domains, \textit{e.g.}, in e-commerce transactions, queries are highly personalized and often involve products from various domains such as instruments, sports, and luxury goods.
Unfortunately, accommodating diverse queries would inevitably suffer from the dilemma of forgetting the general knowledge for CMR tasks.
(b) \textbf{Observations.} 
we study the QS problem for the image-text, video-audio, and composed image retrieval and reveal the following observations:
i) QS not only diminishes the uniformity of queries, but also amplifies the gap and weakens the consistency between the query and gallery sets, both of which would undermine the well-structured common space inherited from source models;
ii) simply employing the vanilla TTA method (\textit{e.g.}, Tent) on diverse queries would result in an obtuse angle between the gradients with respect to domain-specific and pre-training data.
In other words, the harmony between domain-specific knowledge and general knowledge becomes fragile during the adaptation process.
}
\label{fig: fig1}
\end{figure*}
\section{Introduction}
Cross-modal retrieval (CMR)~\cite{Yang_2022_CVPR,SCAN} aims to learn a well-established common space so that the semantically relevant candidates from the gallery set could be prioritized for the given queries, supporting numerous applications such as recommendation systems and search engines.
Recently, the general-to-customized paradigm driven by pre-trained models has emerged as the dominant approach for different CMR tasks such as image-text retrieval~\cite{ALBEF,MA}, video-audio retrieval~\cite{v-a,cav-mae}, and composed image retrieval~\cite{sprc,TME}.
Specifically, the pre-trained model first acquires general knowledge from large-scale data, which bridges the gap across heterogeneous modalities and facilitates the establishment of the common space. 
After that, the customization approaches, including but not limited to domain adaptation~\cite{da1,Prograd}, fully fine-tuning~\cite{BLIP,finetune1}, and low-rank adaptation~\cite{lora,moka}, enable reconciling the distribution shift problem between the general domain and the downstream domains (\textit{e.g.}, natural scenery domain or e-commerce domain).
As a result, the CMR models tailored for different domains could be derived.

Although the general-to-customized paradigm has achieved remarkable success in CMR, existing approaches overlook the inherent complexity of real-world queries and thus inevitably suffer from the query shift (QS) problem.
Specifically, QS exhibits the following two characteristics:
i) \textit{Online Shift}: real-world queries always arrive in an online manner and thus form the query stream, making it infeasible to collect the entire query set beforehand for customization approaches (\textit{e.g.}, domain adaptation).
Moreover, the forms of query vary substantially among different CMR tasks, exacerbating the complexity of the online shift challenge.
For example, as shown in Fig.~\ref{fig: fig1} (a), in video-audio retrieval, the video query might encounter the multiple distribution shifts of gaussian noise and motion blur when rapidly moving in the low-light conditions, while in composed image retrieval, both the reference image and textual modification might suffer from distribution shifts.
As verified in Table~\ref{tab: coco-o-image}–\ref{tab: qgs} and Supplementary Materials, the online shift challenge significantly undermines the effectiveness of the general-to-customized paradigm in various CMR tasks;
ii) \textit{Diverse Shift}: 
as the saying goes, ``Different strokes for different folks'', even with the domain customization, the CMR models cannot always meet the highly-diverse queries from different users or scenarios.
Clearly, tailoring CMR models for each user or scenario to address the above diverse shift challenge is exhausted and infeasible.
Instead, it is highly expected to continually acquire general knowledge from online and diverse queries, so that the CMR models could achieve self-evolution in different scenarios.
However, the self-evolution process of CMR models is susceptible to the overfitting issue on the highly-diverse query data and thus forget the general knowledge that is indispensable for establishing common space between heterogeneous modalities.

To deal with the QS problem, the most promising solution might be the Test-Time Adaptation (TTA) paradigm~\cite{tta_survey1,tta_survey2}, which could reconcile distribution shifts by updating the source model in an online manner.
However, existing TTA methods are intractable for tackling the QS problem due to the following reasons:
i) most TTA methods are specifically designed for the recognition task, which are inadequate for the CMR task that requires finding needles in a haystack. Intuitively, $N$-way matching in the CMR task is considerably harder than $K$-way classification in the recognition task, where $N$ and $K$ denote the number of candidate samples and categories with $N \gg K$;
ii) no pioneer work has investigated the negative impacts of either uni-distribution shift or the more challenging multi-distribution shift on CMR models, thus failing to address the QS problem in various CMR tasks;
iii) nearly all of them focus on accommodating data within a specific domain, while overlooking that real-world queries might originate from various domains, not to mention achieving the expected self-evolution capacity that incrementally learning general knowledge from diverse queries.

To specifically develop an online TTA solution to handle the QS problem for CMR, we present two core observations as illustrated in Fig.~\ref{fig: fig1} (b).
On the one hand, QS would disrupt the distribution of queries and hinder the alignment between the query and gallery sets. 
Specifically, QS would diminish the uniformity of the queries, prohibiting discrimination between queries in the common space. 
Moreover, QS would not only amplify the gap between the query and gallery sets, but also weaken the query-gallery consistency so that the desirable candidates cannot be recalled for the given query, both undermining the well-constructed common space established by the source models.
On the other hand, simply performing the vanilla TTA method on diverse queries may steer the model toward acquiring knowledge that conflicts with general knowledge, thereby undermining the retrieval ability inherited from the pre-trained model.

Based on the above observations, we propose a robust and harmonious TTA method for CMR against QS, dubbed Robust adaptation with quEry ShifT (REST), which consists of a query prediction refinement module, a QS-robust objective, and a gradient decoupling module.
To be specific, the query prediction refinement module refines the retrieval results to formulate the predictions for queries, thereby breaking the dilemma of finding needles in a haystack in CMR.
After that, a novel QS-robust objective function is employed on the refined query predictions to mitigate the negative impacts of QS, which is composed of the following three individual losses:
query uniformity loss performs contrast between queries and their respective centers, thus guaranteeing the discrimination between queries;
query-gallery gap loss narrows the difference between the query and gallery sets with the plausible constraint estimated from the off-the-shelf pre-trained models, thus inheriting the well-established common space; 
query-gallery consistency loss enhances the consistency between query and gallery sets by prioritizing confident query-candidate pairs and alleviating noisy ones with a self-adaptive threshold.
Beyond the designs for addressing online shift challenge, we further propose a gradient decoupling module to achieve harmonious adaptation against the diverse shift challenge.
In brief, during the adaptation process, REST dexterously manipulates gradients to prevent forgetting general knowledge and over-fitting to any specific domain, leading to a continually evolving CMR model.
In summary, the major contributions and novelties of this work could be summarized as follows:
\begin{itemize}
    \item We study a new problem in CMR, termed Query Shift (QS). 
    QS refers to the online query stream that follows a different distribution against the source domain and always exhibits diverse shift types.
    To the best of our knowledge, there are no prior CMR works toward handling the online shift challenge of queries, not to mention the study of the diverse shift challenge.
    \item We reveal the underlying negative impacts of the QS problem on various CMR tasks. 
    In brief, QS would not only undermine the well-established common space derived from the source model, but also steer the model toward acquiring domain-specific knowledge that conflicts with general knowledge, both of which would significantly degrade the CMR performance.
    \item To achieve robustness against QS, we propose a novel test-time adaptation method named REST. In brief, REST first formulates the query predictions to facilitate the existing recognition-oriented TTA methods for CMR. After that, REST adopts a QS-robust objective function to preserve the well-constructed common space by manipulating the core properties in CMR.
    Furthermore, REST employs a novel gradient decoupling module to regularize the adaptation process, thus preventing the retrieval model from forgetting the general knowledge.
    \item To comprehensively investigate the influence of the QS problem and evaluate the solutions, we construct 20 benchmarks featuring either synthetic corruptions or real-world distribution shifts.
    Extensive experiments conducted on these benchmarks demonstrate the effectiveness of the proposed REST across three representative CMR tasks: image–text, video–audio, and composed image retrieval.
\end{itemize}

\section{Related Works}
In this section, we provide a brief review of three topics highly related to this work, including unsupervised domain adaptation for cross-modal retrieval, test-time adaptation, and continual learning for cross-modal retrieval.

\subsection{Unsupervised Domain Adaptation for Cross-modal Retrieval}
The success of the existing CMR methods lies in the identical distribution assumption, \textit{i.e.}, the given test-time queries follow the same distribution as the source domain data.
Unfortunately, it is daunting and even impossible to hold the ideal assumption in real-world scenarios, thus inevitably leading to the distribution shift problem. 
Recently, some Unsupervised Domain Adaptation (UDA) methods have been proposed to reconcile the distribution shift without access to the source data, which has been widely used in CMR tasks.
Based on the way to achieve robustness against distribution shift, existing CMR-oriented UDA approaches could be roughly grouped into the following three categories: 
i) contrastive methods~\cite{UDACVR,DADA}, which adapt the source model to the target domain by maximizing the similarity between the query and its predicted positive candidate in the gallery while minimizing the similarity to negative ones;
ii) regularization methods, which estimate constraints from the source domain, \textit{e.g.}, prototypes~\cite{ACP} and covariance~\cite{CORAL}, and incorporate them as regularizations to guide the domain adaptation;
iii) domain alignment methods~\cite{DASG}, which mitigate the discrepancy between the target and source domains through approaches such as the maximum mean discrepancy minimization or mutual information maximization.

Despite the promising performance, existing CMR-oriented UDA methods require accessing the entire target-domain data.
As a result, these methods are infeasible for handling online query streams, severely limiting their practicality in real-time scenarios such as search engines.
Different from them, we propose a novel test-time adaptation approach for the CMR tasks, advancing CMR toward real-world applications.

\subsection{Test-time Adaptation}
Test-time Adaptation (TTA) has emerged as a dominant paradigm to reconcile the distribution shifts, which has been successfully applied to achieve online domain adaptation.
Based on the motivations of the methods, the existing TTA approaches could be coarsely divided into the following three categories:
i) online TTA methods~\cite{COTTA,DeYO}, which continually update the normalization layers by resorting to the unsupervised objectives, such as entropy minimization or its variants;
ii) robust TTA methods~\cite{EATA,SAR,TTA-Retrieval}, which achieve robust adaptation by filtering out the unreliable predictions with the manually set threshold;
iii) TTA beyond recognition~\cite{READ,MM-TTA,MSA-TTA}, which achieves online domain adaptation for the various tasks, including but not limited to segmentation~\cite{MM-TTA}, and sentiment analysis~\cite{MSA-TTA}.

Among existing approaches, the multi-modal TTA method READ~\cite{READ} is most relevant to our work, while embracing different motivations from our work.
In brief, READ aims to mitigate the negative impacts of unreliable modalities on multi-modal recognition, while REST focus on achieving robustness against online and diverse QS in CMR.
The experiments in Section~\ref{sec: experiments} further highlight the necessity of tailoring a novel TTA method for CMR.
Moreover, it should be pointed out that this paper differs from our conference version TCR~\cite{tcr} in three key aspects:
i) \textbf{Problem Extension:}
TCR only considers the online shift in image-text retrieval, while REST delves into QS from a more general and practical perspective.
On the one hand, we pioneer the explorations of multi-distribution shift in video-audio and composed image retrieval.
On the other hand, we further study the diverse shift issue, which refers to queries from different domains and poses a new challenge for the harmonious adaptation of CMR models;
ii) \textbf{Method Improvement:} 
to deal with online and diverse shift challenges, REST not only develops a noise-robust loss to enhance query-gallery consistency, but also proposes a novel gradient decoupling module to prevent the CMR model from forgetting the general knowledge;
iii) \textbf{Application Expansion:} 
REST is capable of handling a wider spectrum of CMR tasks and supports online adaptation across more CMR models, including single-stream (\textit{e.g.}, BLIP~\cite{BLIP}) and dual-stream (\textit{e.g.}, CLIP~\cite{CLIP}) CMR models.

\subsection{Continual Learning for Cross-modal Retrieval}
Continual learning aims to acquire new knowledge from the dynamic data distributions throughout the lifetime.
The key challenge for continual learning is general knowledge forgetting, \textit{i.e.}, adapting to a new domain would result in forgetting the knowledge previously learned from the source domain.
Based on the way to mitigate general knowledge forgetting, the existing continual learning methods for CMR tasks could be divided into the following three categories:
i) replay-based methods~\cite{CL_replay1}, which approximate and recover source-domain distributions by storing source-domain data or training a generative model;
ii) constraint-based methods~\cite{CL-contraint}, which employ penalty terms to explicitly constrain the variation of network parameters;
iii) architecture-based methods~\cite{CL-architecture1,CL-architecture2}, which allocate dynamic network parameters to store domain-specific knowledge.

The major differences between existing continual learning approaches and this work are given below. 
On the one hand, existing works focus on learning from a sequence of domains one by one, whereas our work addresses a more challenging and practical scenario, \textit{i.e.}, diverse queries from various domains.
On the other hand, existing works would lose either generalizability or effectiveness, while our work seeks to embrace the advantages of both.
Specifically, the replay- and architecture-based methods heavily rely on additional source-domain data and domain-specific networks, while the constraint-based methods strive for invariant parameters and thus hinder the acquisition of new knowledge.
In contrast, our work not only updates the parameters without requiring access to the source-domain data, but also continually acquires and accumulates general knowledge by self-adaptively eliminating gradients that conflict with general knowledge.

\begin{figure*}[t]
\centering
\includegraphics[width=0.95\linewidth]{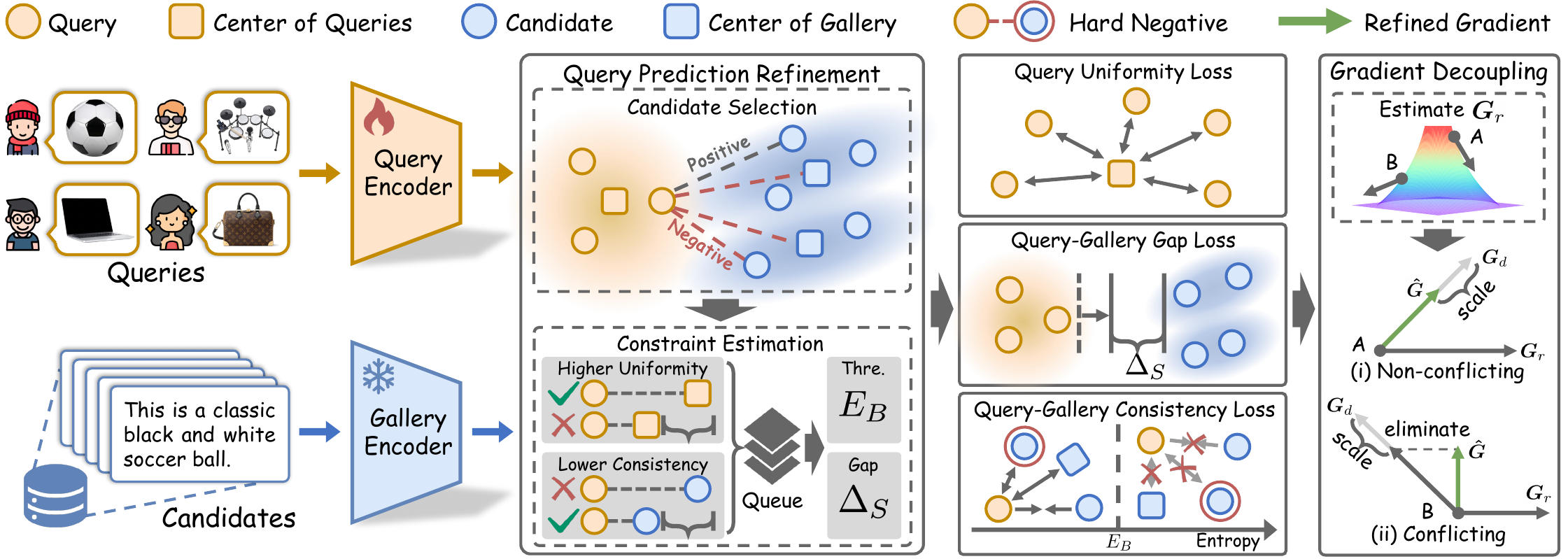}
\caption{
Overview of the proposed REST. 
For the given online queries from diverse domains, the query and gallery encoders are first adopted to map the queries and candidates into the common space.
The obtained embeddings are passed into the query prediction refinement module, which selects positive and valuable negative pairs for each query to formulate the refined query prediction.
After that, the positives with higher uniformity and lower consistency are adopted to estimate the threshold for noise filtering and the query-gallery gap that constrains the adaptation process.
Then, three independent losses are employed to achieve robust adaptation against QS.
Finally, the gradient decoupling module manipulates the strength and direction of the gradient to achieve harmonious adaptation, thus avoid forgetting the general knowledge.
}
\label{fig: fig2}
\end{figure*}

\section{Method}
In this section, we propose a novel test-time adaptation method for CMR tasks, dubbed Robust adaptation with quEry Shift (REST).
The section is structured as follows.
In Section~\ref{sec: problem formulation}, we present the formal
definition of the query shift problem and propose a simple baseline to make existing TTA methods feasible for CMR.
In Section~\ref{sec: Query Prediction Refinement}, we employ a query prediction refinement module to formulate query predictions and thus avoid finding needles in a haystack for CMR.
In Section~\ref{sec: Query Uniformity Learning}-\ref{sec: Query-Gallery Consistency Learning}, we design a novel objective function to achieve robust CMR against QS.
In Section~\ref{sec: Gradient Decoupling}, we introduce a gradient decoupling module to mitigate the negative impact brought by the diverse queries.

\subsection{Notations and Problem Formulation}
\label{sec: problem formulation}
Let $\mathcal{D}_{T} = \left\{\mathbf{X}^{Q}=\{x_{i}^{Q}\}_{i=1}^{N^{Q}}, \mathbf{X}^{G}=\{x_{j}^{G}\}_{j=1}^{N^{G}}\right\}$ denotes the target-domain data and $f_{\Theta_{S}}$ denotes the model pre-trained on the source-domain data $\mathcal{D}_{S}$, which consists of the two encoders, \textit{i.e.}, $f_{\Theta^{Q}_{s}}$ and $f_{\Theta^{G}_{s}}$.
For a given uni-modal, temporal, or multi-modal query $x^{Q}_{i}\in \mathcal{D}_{T}$, retrieval aims to associate the corresponding sample (\textit{i.e.}, candidate) $x^{G}_{j}\in \mathbf{X}^{G}$, where $Q$ and $G$ denote as query set and gallery set. 
However, the success of existing retrieval methods heavily relies on the identical distribution assumption, which is easily violated in real-world applications and thus leads to the query shift problem.
\begin{definition}[Query Shift] 
    For the source model $f_{\Theta_{S}}$ pre-trained on the source-domain data $\mathcal{D}_{S}$, query shift refers to the online query stream $\mathbf{x}^{Q}$ originating from the target-domain data $\mathcal{D}_{T}$ follows distinct distribution with $\mathcal{D}_{S}$, \textit{i.e.}, $\mathcal{P}\left(\mathcal{D}_{T}\right) \not\sim \mathcal{P}\left(\mathcal{D}_{S}\right)$, where $\mathcal{P}\left(\mathcal{\cdot}\right)$ denotes the distribution of the data.
\end{definition}
As a result, such a query shift problem would lead to the serious performance degradation of $f_{\Theta_{S}}$, which is further verified in the experiments. 
To enable $f_{\Theta_{S}}$ for online adaptation with query shift, the test-time adaptation (TTA) paradigm might offer a promising solution.
As discussed in Introduction, most existing TTA approaches are carefully designed for the recognition task, leaving an urgent need to bridge the gap between recognition and CMR tasks.
To endow the recognition-oriented TTA approaches with the capacity to tackle the CMR tasks, we first formulate the retrieval task as a query prediction process in analogy with the recognition task that assigns given samples to their corresponding categories.
Formally, for a given online batch of queries $\mathbf{x}^{Q}$ with size $B$, the corresponding query predictions is defined as,
\begin{equation}
\mathbf{p}=\operatorname{Softmax}\left(\mathbf{z}^{Q}\left(\mathbf{Z}^{G}\right)^{T}/\tau \right),
    \label{eq: prediction}
\end{equation}
where $\tau$ is the temperature, $\mathbf{z}^{Q}=f_{\Theta^{Q}_{s}}(\mathbf{x}^{Q})\in \mathbb{R}^{B\times D}$ and $\mathbf{Z}^{G}=f_{\Theta^{G}_{s}}(\mathbf{X}^{G})\in \mathbb{R}^{N^{G}\times D}$ are the $\ell 2$-normalized $D$-dimension embeddings for the given queries and candidates, respectively.

Thanks to the above formulation, most existing TTA approaches could be employed to tackle the query shift challenge with the following objective, 
\begin{equation}
        \min_{\tilde{\Theta}} \mathcal{L}_{TTA}\left(\mathbf{p}\right),
    \label{eq: tta}
\end{equation}
where $\mathcal{L}_{TTA}$ denotes the unsupervised objectives (\textit{e.g.}, entropy minimization) of existing TTA approaches, $\tilde{\Theta} \subseteq \Theta_{S}$ denotes the learnable parameters. 
However, unlike recognition tasks that typically involve a limited number of candidate classes, retrieval tasks aim to identify the most relevant counterpart for a given query from a large number of candidates in the gallery.
As a result, such a simple formulation struggles to meet the ``finding needles in a haystack" requirement of CMR tasks, thus achieving suboptimal effects as verified in Fig.~\ref{fig: ablation} (a).
Besides, as discussed in Introduction, such a simple baseline overlooks the underlying negative impacts of the query shift challenge and thus fails to achieve promising performance.
Specifically, on the one hand, the existing methods are unable to explicitly restore the well-structured common space inherited from $f_{\Theta_{S}}$.
On the other hand, these methods struggle to deal with the general knowledge forgetting problem induced by the diverse queries.

To tackle the query shift challenge, we propose Robust adaptation with quEry Shift (REST), which consists of the query prediction refinement module, the QS-robust objective function, and the gradient decoupling module.
As shown in Fig.~\ref{fig: fig2}, for the given batch of diverse online queries, REST first employs a novel query prediction refinement module to obtain the TTA-favorable query prediction $\hat{\mathbf{p}}$ and estimate the constraints to support the optimization of the objective function.
After that, we employ the following QS-robust objective function to achieve robust retrieval against query shift, \textit{i.e.},
\begin{equation}
\min_{\tilde{\Theta}} \mathcal{L}\left(\hat{\mathbf{p}}\right),
    \label{eq: overall}
\end{equation}
where $\mathcal{L}= \mathcal{L}_{U}+\mathcal{L}_{G}+\mathcal{L}_{C}$, with $\mathcal{L}_{U}$, $\mathcal{L}_{G}$, and $\mathcal{L}_{C}$ denoting the query uniformity loss, query-gallery gap loss, and query-gallery consistency loss, respectively. 
Beyond the designs of the objective function, we propose a novel gradient decoupling module to mitigate the general knowledge forgetting problem caused by diverse queries.
In the following, we will elaborate on each module and loss individually.

\subsection{Query Prediction Refinement}
\label{sec: Query Prediction Refinement}
In this section, we introduce the query prediction refinement module, which selects candidates for each given query to refine the prediction and estimates the constraints to facilitate the optimization of Eq.~\ref{eq: overall}.
\subsubsection{Candidate Selection}
Although the formulated query predictions in Eq.~\ref{eq: prediction} enable the recognition-oriented TTA approaches to address the CMR tasks, employing the vanilla query prediction in Eq.~\ref{eq: prediction} would result in either model underfitting or overfitting.
Specifically, due to the variable and substantial number of candidates, obtaining trustworthy query prediction for TTA is akin to finding needles in a haystack, resulting in model underfitting.
Although one straightforward solution is to employ a low temperature for further distinguishing the retrieval results in the query predictions, such a solution tends to cause $f_{\Theta_{S}}$ overfitting on noisy query predictions and thus leads to performance degradation.

To break the above dilemma, we propose to construct a subset of candidates and then establish new query predictions for the given queries. 
Specifically, for a given query $x^{Q}_{i}$ in the mini-batch, the corresponding candidates are obtained via
\begin{equation}
    \mathbf{x}^{G^{\prime}}_{i} = \Big[ 
    \underbrace{\mathcal{N}_{1}(x^{Q}_{i})}_{\text{positive}}, \ \  
    \underbrace{\cup_{j \neq i}\mathcal{N}_{K}(x^{Q}_{j}) \vphantom{\mathcal{C}(\mathbf{X}^{G})}}_{\text{sample negatives}},
    \underbrace{\mathcal{C}_{K}(\mathbf{X}^{G})}_{\text{cluster negatives}}
    \Big],
    \label{eq: new gallery}
\end{equation}
where $\mathcal{N}_{K}(\cdot)$ denotes the selected subset and $\mathcal{C}_{K}(\cdot)$ denotes the centroids of the gallery set, with $K$ indicating the number of selected samples or centroids.
In practice, $\mathcal{N}_{K}(\cdot)$ and $\mathcal{C}(\cdot)$ are implemented by query-to-candidate nearest neighborhood selection and $k$-means clustering, respectively.
Consequently, the refined query prediction could be formulated as follows,
\begin{equation}
    \hat{p}_{i}=\operatorname{Softmax}\left(z_{i}^{Q} \left(\mathbf{z}_{i}^{G^{\prime}}\right)^{T}/ \tau \right),
    \label{eq: new prediction}
\end{equation}
where $\mathbf{z}_{i}^{G^{\prime}}=f_{\Theta^{G}_{s}}(\mathbf{x}^{G^{\prime}}_{i})$, $\mathbf{\hat{p}}=[\hat{p}_{1},\hat{p}_{2},\cdots,\hat{p}_{B}]$ denotes the refined query predictions for the online-batched queries $\mathbf{x}^{Q}$.
The query prediction refinement manner embraces the following merits:
i) the valuable negatives would facilitate the consistency learning as described in Section~\ref{sec: Query-Gallery Consistency Learning}, which has been widely recognized in prior works~\cite{simclr,moco}.
Specifically, the centroids of the gallery set could be treated as the representative negatives, while the Top-$K$ selection manner could enhance the diversity of negatives;
ii) query prediction refinement module excludes some irrelevant samples in the gallery, thus preventing model overfitting to some extent.
iii) the excluded samples would avoid finding needles in a haystack for queries, thus alleviating the model underfitting.

\subsubsection{Constraint Estimation}
Some pioneer works~\cite{CL-contraint,CAN} have empirically found that the source-domain data could effectively constrain the domain adaptation process, thereby circumventing the general knowledge forgetting problem.
However, the source-domain data is always unavailable in real-world CMR applications.
To remedy this, we propose selecting some source-domain-like data from the positive query-candidate pairs to estimate desirable constraints that facilitate the adaptation process.
Specifically, we adopt a carefully-designed principle to select the positive pair $(x_{i}^{Q},x_{j}^{G})$ with $x_{j}^{G}=\mathcal{N}_{1}(x_{i}^{Q})$, formally,
\begin{equation}
    s_{i}=2\left(\|z_{i}^{Q}-z_{j}^{G}\|\right)-\left(\|z_{i}^{Q}-\overline{\mathbf{z}}^Q\|+\|z_{j}^{G}-\overline{\mathbf{z}}^{G}\|\right),
    \label{eq: principle}
\end{equation}
where $\overline{\mathbf{z}}^Q=\frac{1}{B}\sum_{i}^{B}z_i^Q$ and $\overline{\mathbf{z}}^{G}=\frac{1}{B}\sum_{i}^{B}f_{\Theta^{G}_{s}}(\mathcal{N}_{1}(x_{i}^{Q}))$ are the centers of the queries and the corresponding candidates in the common space, respectively.
In the implementation, we preserve a dynamic queue with size $B$ to store the embeddings $(\mathbf{z}^{Q_{B}},\mathbf{z}^{G_{B}})$ of the source-domain-like data, \textit{i.e.}, query-candidate pairs with the smallest $s_{i}$.
As illustrated in Fig.~\ref{fig: fig1} (b), the source-domain query-candidate pairs embrace small consistency and high uniformity, so the query-candidate pairs with low principle have higher probabilities to be source-domain-like data.

With the constructed source-domain-like data, we propose to estimate the gap between query and gallery sets in the source domain, which is further used in the query-gallery gap learning (Section~\ref{sec: Query-Gallery Gap Learning}). 
Mathematically,
\begin{equation}
    \Delta_{S}=\left\| \frac{1}{B}\sum_{i}^{B}z^{Q_{B}}_{i}- \frac{1}{B}\sum_{j}^{B}z^{G_{B}}_{j} \right\|.
    \label{eq: estimate query-gallery gap}
\end{equation}
As another by-product, a desirable threshold for identifying the noisy $\hat{p}$ could be adaptively determined as follows,
\begin{equation}
    E_{B}=\max_{i=1, \dots, B}  E\left( x_{i}^{Q_{B}} \right),
\label{eq: entropy threshold}
\end{equation}
where $E\left(\cdot\right)$ denotes the entropy of the given query and $x_{i}^{Q_{B}}$ denotes the $i$-th query in the source-domain-like data.

\subsection{Query Uniformity Learning}
\label{sec: Query Uniformity Learning}
As discussed in Introduction, query shift would diminish the uniformity of the queries, resulting in queries with low discrimination in the common space.
Consequently, the retrieval model struggles to distinguish among the confused queries, leading to severe performance degradation.
To solve the problem, we propose a novel query uniformity loss as follows,
\begin{equation}
    \mathcal{L}_{U} = \frac{1}{B}\sum_{i}^{B}exp\left(-\|z_i^Q - \overline{\mathbf{z}}^Q\| \right).
    \label{eq: loss uniformity}
\end{equation}
Such a behavior could impose the contrasts between queries and their respective centers, thereby explicitly enhancing the uniformity of the queries.

\subsection{Query-Gallery Gap Learning}
\label{sec: Query-Gallery Gap Learning}
Apart from the negative impacts on the distribution of queries, query shift would amplify the gap between query and gallery sets, undermining the well-constructed query-gallery alignment established by the source models.
To address the problem, one feasible solution is to restore the query-gallery gap in the source domain.
Thanks to the estimated constraint $\Delta_{S}$ in Eq.~\ref{eq: estimate query-gallery gap}, we propose the following query-gallery gap learning loss, \textit{i.e.},
\begin{equation}
    \mathcal{L}_{G}=\left(\Delta_{T} - \Delta_{S} \right)^{2},
    \label{eq: loss_mmg}
\end{equation}
where $\Delta_{T}=\left\|\overline{\mathbf{z}}^Q-\overline{\mathbf{z}}^{G}\right\|$ denotes the query-gallery gap in the target domain.
The key idea behind $\mathcal{L}_{G}$ is to inherit the well-constructed common space of the source model by rectifying the query-gallery gap, thus achieving robustness against QS.
Notably, as illustrated in Fig.~\ref{fig: observation}, over-eliminating the query-gallery gap would not boost or even degrade the performance of the retrieval model.
A similar phenomenon has also been preliminarily observed in multi-modal representation learning~\cite{MindGap}, which further verifies that rectifying the query-gallery gap to a plausible constraint is reasonable.

\subsection{Query-Gallery Consistency Learning}
\label{sec: Query-Gallery Consistency Learning}
For the CMR tasks, it is widely acknowledged that the queries and their corresponding candidates should exhibit consistency in the common space~\cite{MA,cross_modal4,yin1}.
However, query shift would diminish the consistency between query and gallery sets, thus preventing the queries from associating with the correct candidates. 
To remedy this, we propose the following query-gallery consistency loss,
\begin{equation}
    \mathcal{L}_{C}=\mathcal{L}_{REM}+\mathcal{L}_{RHM},
\end{equation}
where $\mathcal{L}_{REM}$ denotes the robust entropy minimization loss, $\mathcal{L}_{RHM}$ denotes the robust hard mining loss. 
In the following, we will elaborate on each loss individually.
\subsubsection{Robust Entropy Minimization}
As the most widely used paradigm to achieve consistency between two heterogeneous sets, contrastive loss aims to maximize the similarity of positive pairs while minimizing that of negative pairs.
The key to contrastive loss lies in constructing the valuable negatives for each positive, thus facilitating the identification between associated and mismatched pairs.
Thanks to the query predictions refinement module in Section~\ref{sec: Query Prediction Refinement}, we have chosen the representative and diverse negatives for each given positive pair.
However, due to the unavailability of ground-truth correspondences, the estimated positive pairs in Eq.~\ref{eq: new prediction} might be mismatched, resulting in noisy query predictions.

To achieve robustness against the noisy query predictions, we propose the following robust entropy minimization loss,
\begin{equation}
    \begin{gathered}
        \mathcal{L}_{REM}=\frac{1}{\sum_{i} \mathbb{I}_{\{W(x_{i}^{Q})\neq 0\}}} \sum_{i=1}^{B} W(x_{i}^{Q}) E(x_{i}^{Q})\\
        W(x_{i}^{Q})=\max \left(1-E(x_{i}^{Q})/E_{B}, 0\right)
    \end{gathered},
    \label{eq: loss_rem}
\end{equation}
where $E_{B}$ is the self-adaptive threshold estimated in Eq.~\ref{eq: entropy threshold}, and $\mathbb{I}_{\{\cdot\}}$ is an indicator function evaluating to $1$ \textit{i.f.f.} the condition is satisfied. 
Such behavior could achieve robustness against noise by excluding unreliable query predictions from adaptation and assigning higher weights to trustworthy query predictions.

\subsubsection{Robust Hard Mining}
Although the entropy-based objective in Eq.~\ref{eq: loss_rem} could achieve consistency between query and gallery sets to some extent, it is inadequate for distinguishing the discrepancy between positives and the hard negatives, resulting in suboptimal query-gallery consistency.
\begin{theorem}
    Let $\hat{p}$ denotes the refined query prediction of the given query $x$, where $\hat{p}_{i}$ denotes the probability that $x$ is associated with the $i$-th candidate. For the entropy minimization objective $\mathcal{L}_{EM}=-\sum_{i}\hat{p}_{i}\log \hat{p}_{i}$, the gradient with respect to the probability of easy negative $\hat{p}_{m}$ exceeds that of the hard negative $\hat{p}_{n}$ whenever $\hat{p}_{m}< \hat{p}_{n} \leq \frac{1}{e}$. 
    Mathematically, we have
    \begin{equation}
        \left|\frac{\partial \mathcal{L}_{EM}}{\partial \hat{p}_{m}}\right| > \left|\frac{\partial \mathcal{L}_{EM}}{\partial \hat{p}_{n}}\right|
        \label{eq: theorem_entropy}
    \end{equation}
    where $|\cdot|$ denotes the absolute value.
    \label{theorem: entropy minimization}
\end{theorem}
The according proof is presented in Supplementary Materials.
Theorem~\ref{theorem: entropy minimization} illustrates that the entropy minimization paradigm tends to optimize easy negatives rather than the hard ones. 
Note that the condition $\hat{p}_{m}< \hat{p}_{n} \leq \frac{1}{e}$ is easily satisfied in the retrieval task, as the substantial candidates would result in low probabilities for negatives.

To further enhance the consistency between query and gallery sets, we propose the following robust hard mining loss to distinguish the differences between positive and hard negatives in a self-adaptive manner, \textit{i.e.},
\begin{equation}
    \begin{gathered}
        \mathcal{L}_{RHM}=\frac{1}{\sum_{i} \mathbb{I}_{\{W(x_{i}^{Q})\neq 0\}}}\sum_{i}^{B}W(x_{i}^{Q})H(x_{i}^{Q})\\
         H(x_{i}^{Q})=-\log c_{ij} + \log c_{ik}
    \end{gathered},
    \label{eq: loss_rhm}
\end{equation}
where $c_{ij}\in [0,1]$ denotes the consistency of the estimated positive and $c_{ik}\in [0,1]$ (\textit{i.e.}, $k\neq j$) denotes the consistency of the corresponding hard negative, with larger values indicating stronger semantic association.
In the implementation, the consistency could be derived by either the cosine similarity from the dual-stream network or the matching score from the single-stream network. 
See more details in Supplementary Material~\textsc{IV}-E.

\subsection{Gradient Decoupling}
\label{sec: Gradient Decoupling}
As discussed in Introduction, directly performing TTA on the diverse queries might mislead the model towards forgetting the general knowledge.
To be specific, we first give a formal definition of the gradients with respect to the general knowledge from the source model during the adaptation, \textit{i.e.}, general direction. Formally,
\begin{definition}
    General direction refers to the gradient of the Kullback-Leibler (KL) divergence between the predictions of the source model and the TTA-tuned model. Mathematically,
    \begin{equation}
        \begin{gathered}
            \boldsymbol{G}_{r}=\nabla D_{KL}, \ \
            D_{KL}=-\frac{1}{B}\sum_{i=1}^{B}\sum_{j}\hat{p}_{ij}^{S}\log \frac{\hat{p}_{ij}}{\hat{p}_{ij}^{S}},
        \end{gathered}
    \end{equation}
    where $\hat{p}_{ij}^{S}$ denotes the refined query prediction of the source model $f_{\Theta_{S}}$.
    \label{def: general knowledge}
\end{definition}
Intuitively, the well-pretrained source model has already mined general knowledge from large-scale data, whose zero-shot query prediction could reflect the general knowledge to some extend.
Similarly, we give the formal definition of gradients with respect to the domain-specific knowledge during the adaptation, \textit{i.e.}, domain-specific direction. Formally,
\begin{definition}
    Domain-specific direction refers to the gradient of the objective function $\mathcal{L}_{TTA}$ during the adaptation process, \textit{e.g.}, entropy-based objectives,
    \begin{equation}
        \boldsymbol{G}_{d}=\nabla \mathcal{L}_{TTA}.
    \end{equation}
    \label{def: domain-specific knowledge}
\end{definition}
However, as illustrated in Fig.~\ref{fig: fig1}(b) and Table~\ref{tab: gra_effect}, the domain-specific direction $\boldsymbol{G}_{r}$ might conflict with the general direction $\boldsymbol{G}_{d}$, thus resulting in the general knowledge forgetting problem.
To address the problem, we propose a novel Gradient Decoupling (GD) module, which eliminates the gradient conflicting with the general direction and imposes penalties to avoid over-optimization.
Specifically, we decompose the domain-specific direction $\boldsymbol{G}_{d}$ into:
\begin{itemize}
    \item Component $\boldsymbol{G}_{\perp}$ orthogonal to the general direction, which denotes the non-conflicting domain-specific knowledge.
    \item Component $\boldsymbol{G}_{\|}$ parallel to the general direction, which could be the same as $\boldsymbol{G}_{r}$ or the opposite.
\end{itemize}
Clearly, the opposite of $\boldsymbol{G}_{r}$ indicates the conflicting update that should be discarded to avoid general knowledge forgetting.
To this end, we refine the update gradient as follows,
\begin{equation}
    \boldsymbol{\hat{G}}= \begin{cases}W_d \ \boldsymbol{G}_{\perp} + W_d \ \boldsymbol{G}_{\|}, & \text { if } \boldsymbol{G}_{d} \cdot \boldsymbol{G}_{r} \geq 0 \\ 
    W_d \ \boldsymbol{G}_{\perp}, & \text { if } \boldsymbol{G}_{d} \cdot \boldsymbol{G}_{r} < 0 
    \end{cases},
\end{equation}
where $W_d=\sigma(D_{KL})$ denotes the self-adaptive weight that regulates the magnitude of updating,  $\sigma(\cdot)$ indicates an inverse function, and we set $\sigma(x)=exp(-x)$ for simplicity in the implementation.
$\boldsymbol{G}_{d} \cdot \boldsymbol{G}_{r} < 0$ indicates that the optimization direction conflicts with the general knowledge, and vice versa. 

\begin{theorem}
The refined gradient $\boldsymbol{\hat{G}}$ never conflicts with the general direction $\boldsymbol{G}_r$, \textit{i.e.}, $\boldsymbol{\hat{G}}\cdot \boldsymbol{G}_{r} \geq 0$
\label{theorem: gradient}
\end{theorem}
The proof is presented in Supplementary Materials.
The proposed gradient decoupling module would embrace the following two advantages.
On the one hand, according to Theorem~\ref{theorem: gradient}, the refined gradient $\boldsymbol{\hat{G}}$ eliminates the component conflicts with general knowledge, thus preserving the general knowledge.
On the other hand, when $D_{KL}$ enlarges, \textit{i.e.}, the CMR model might forget general knowledge, the self-adaptive weight would impose penalties on the gradients to avoid over-optimization on any particular domain.
Note that a few works have explored harmonious learning in UDA and few-shot scenarios~\cite{gradient,Prograd}, while the proposed REST differs from them in the following aspects:
i) most of them rely on the ideal assumption that the entire test-time set is always available, whereas the REST is designed for online adaptation, making it more flexible for real-world applications.
ii) existing works overlook that continuous adaptation on diverse queries might overfit to the domain-specific data, while REST imposes penalties on the gradients to avoid overfitting issues.

\section{Experiments}
\label{sec: experiments}
To verify the effectiveness of REST for achieving QS-robust retrieval, we conduct expensive experiments on three different CMR tasks, including image-text retrieval, video-audio retrieval, and composed image retrieval.
Due to space limitation, we present more experimental details and results in Supplementary Materials.

\begin{table}[htbp]
\centering
\caption{Configurations of settings. ``General'' denotes the pre-training data, ``Specific'' denotes the down-streaming data, ``Corrupt'' denotes the manually injected corruptions. ``Online'' indicates whether the test-time queries arrive in an online manner, and ``Diverse'' indicates whether the test-time queries are derived from various domains.}
\label{tab: settings}
\Large
\resizebox{0.9\linewidth}{!}{
\begin{tabular}{c|cc|cc}
\multicolumn{1}{c|}{\textbf{Settings}} & \multicolumn{1}{c}{\textbf{Source}} & \multicolumn{1}{c|}{\textbf{Target}} & \multicolumn{1}{c}{\textbf{Online}} & \multicolumn{1}{c}{\textbf{Diverse}} \\ \hline
\multirow{2}{*}{OQS} & General  & Specific              & $\checkmark$ & $\times$ \\ \cline{2-5}
                              & Specific & Specific + Corrupt & $\checkmark$ & $\times$ \\ \hline
\multirow{2}{*}{DQS} & General  & Specific              & $\checkmark$ & $\checkmark$ \\ \cline{2-5}
                              & Specific & Specific + Corrupt & $\checkmark$ & $\checkmark$ \\ \hline
\end{tabular}
}
\end{table}

\subsection{Experiment Configurations}
In this section, we elaborate on the experiment configurations of REST, including settings, benchmarks, and implementation details.

\textbf{Settings:}
To investigate the influence of CMR with QS, we employ the Online Query Shift (OQS) and the Diverse Query Shift (DQS) settings for extensive evaluations, each setting involves two types of distribution shift, \textit{i.e}, simulated one with widely-used corruptions and real-world one varying from different domains.
As shown in Table~\ref{tab: settings}, for the real-world distribution shift, we regard pre-training and down-streaming data as the source and target domain, respectively.
While for the simulated distribution shift, with the down-streaming data spanning the source domain, we inject synthetic corruptions into the queries from the down-streaming data and thus obtain the target domain data.
Note that the simulated distribution shift only occurs on the query set, while the real-world one might simultaneously exhibit on both the query and gallery sets.

\textbf{Benchmarks:}
To facilitate the investigation for TTA with query shift, we construct 20 benchmarks featuring either synthetic corruptions or real-world distribution shifts.
To be specific, ten corruption-injected benchmarks are derived by introducing OQS and DQS settings on five widely-used CMR datasets, including COCO~\cite{COCO}, Flickr~\cite{Flickr}, AudioSet~\cite{audioset}, FIQ~\cite{FashionIQ}, and CIRR~\cite{cirr} datasets.
More specifically,
\begin{itemize}
    \item \textbf{Image-text Retrieval}: we introduce 16 types of image corruptions and 15 types of text corruptions across COCO and Flickr datasets, which exhibit one type of distribution shift on either the image or text modality.
    \item \textbf{Video-audio Retrieval}: we introduce 12 types of video corruptions and 6 types of audio corruptions in the AudioSet dataset, which exhibits multiple types of distribution shift on either the video or audio modality.
    \item \textbf{Composed Image Retrieval}: we introduce 16 types of corruptions to the reference images and 15 types of corruptions to the textual modifications in the FIQ and CIRR datasets, which exhibit distribution shifts on both the image and text modalities.
\end{itemize}
For benchmarks with real-world distribution shift, we adopt the following widely used datasets: Flickr, COCO, Fashion-Gen~\cite{Fashion-Gen}, Nocaps~\cite{Nocaps}, RSICD~\cite{RSICD} and RSITMD~\cite{RSITMD} for image-text retrieval, AudioSet and VGGSound~\cite{vggsound} for video-audio retrieval, and FIQ and CIRR for composed image retrieval.
Note that, the benchmarks under the OQS setting are denoted with the suffix ``-O'', while those under the DQS setting are denoted with ``-D''.

\textbf{Implementation Details }
REST is a model-agnostic TTA method that could endow most existing pre-trained models with robustness against the query shift problem.
To support this claim, we adopt BLIP~\cite{BLIP}, CAV-MAE~\cite{cav-mae}, BLIP-2~\cite{blip2} as the pre-trained models, since they are the widely used pre-trained models for image-text, video-audio, and composed image retrieval.
For the evaluation, we adopt the recall metric~\cite{cross_modal5,BLIP} for the various CMR tasks, including image-to-text retrieval (a.k.a. TR), text-to-image retrieval (a.k.a. IR), video-audio retrieval (a.k.a. AR), audio-video retrieval (a.k.a. VR), composed image retrieval (a.k.a. CIR).
Following~\cite{Tent}, REST would continuously update the parameters in the source model $f_{\Theta^{Q}_{s}}$ using online mini-batches of queries, with the batch size set to $32$ for composed image retrieval and $64$ for the other tasks.
To be specific, the learnable parameters in $\tilde{\Theta}$ (Eq.~\ref{eq: overall}) correspond to the Layer Normalization (LN) layers in the implementation.
For the hyperparameter configuration, the temperature $\tau$ in Eq.~\ref{eq: prediction}, and the selected number $K$ of both neighbors and centroids in Eq.~\ref{eq: new prediction} are fixed as $0.02$ and $10$ for all experiments, respectively.

\begin{table*}[t]
    \caption{Comparisons with state-of-the-art methods on COCO-O benchmark under \textbf{\textsc{OQS on image modality}} regarding Recall@1 metric. The best results are marked in \textbf{bold} and the second best results are \underline{underlined}. REST is the journal extension of TCR.}
    \label{tab: coco-o-image}
\newcommand{\tabincell}[2]{\begin{tabular}{@{}#1@{}}#2\end{tabular}}
 \begin{center}
 \begin{threeparttable}
 \Large
    \resizebox{0.98\linewidth}{!}{
 	\begin{tabular}{l|cccc|cccc|cccc|cccc|>{\columncolor{blue!8}}c}
 	\multicolumn{1}{c}{} & \multicolumn{4}{c}{Noise} & \multicolumn{4}{c}{Blur} & \multicolumn{4}{c}{Weather} & \multicolumn{4}{c}{Digital}  \\
 	 Methods & Gauss. & Shot & Impul. &Speckle & Defoc. & Glass & Motion & Zoom & Snow & Frost & Fog & Brit. & Contr. & Elastic & Pixel & JPEG & Avg.  \\
    \cmidrule{1-18}
        BLIP &  43.4 & 46.3 & 43.2 & 57.3 & 43.3 & 68.0 & 39.7 & 8.4 & 32.3 & 52.2 & 57.0 & 66.8 & 36.0 & 41.3 & 20.6 & 63.7 & 45.0 \\ 
        ~~$\bullet~$Tent & 41.6 & 40.5 & 37.9 & 54.0 & 44.7 & 65.1 & 39.6 & 8.3  & 31.9 & 48.7 & 56.3 & 66.5 & 31.8 & 40.3 & 19.2 & 62.3 & 43.0 \\ 
        ~~$\bullet~$PL & 43.2 & 44.9 & 29.9 & 57.3 & 13.3 & 71.8 & 31.0 & 1.1 & 25.0 & 50.5 & 57.9 & 70.3 & 14.4 & 38.5 & 10.4 & 65.8 & 39.1 \\
        ~~$\bullet~$SHOT & 47.2 & 42.9 & 34.4 & 63.8 & 10.8 & 71.9 & 6.8 & 0.5 & 8.3 & 52.2 & 51.8 & 70.6 & 7.9 & 39.8 & 4.7 & 67.7 & 36.3 \\
        ~~$\bullet~$EATA & 41.4 & 50.3 & 35.7 & 63.1 & 49.8 & 72.2 & 46.2 & 6.9  & 45.6 & 56.7 & 62.5 & 71.4 & 43.6 & 51.3 & 25.6 & 67.0 & 49.3  \\
        ~~$\bullet~$SAR & 42.3 & 51.5 & 37.5 & 61.8 & 40.3 & 71.5 & 32.8 & 6.2  & 38.0 & 56.2 & 59.1 & 70.6 & 31.1 & 53.5 & 17.5 & 66.4 & 46.0  \\
        ~~$\bullet~$READ & 45.8 & 48.4 & 37.2 & 59.9 & 44.5 & 71.8 & 46.6 & 11.5 & 39.9 & 49.9 & 58.4 & 70.3 & 35.8 & 45.0 & 18.8 & 66.2 & 46.9  \\
        ~~$\bullet~$COME & 18.9 & 37.1 & 25.5 & 60.2 & 8.8 & 71.4 & 7.8 & 0.6 & 12.7 & 31.3 & 40.3 & 70.4 & 6.8 & 31.2 & 3.3 & 67.7 & 30.9 \\
        ~~$\bullet~$TSA & 40.8 & 43.9 & 42.3 & 61.0 & 38.2 & 70.2 & 37.4 & 11.2 & 32.9 & 49.2 & 62.4 & 69.3 & 40.3 & 39.4 & 21.0 & 64.6 & 45.2 \\
        \rowcolor{pink!30}~~$\bullet~$TCR & \underline{51.5} & \underline{55.4} & \underline{53.8} & \underline{64.7} & \underline{58.3} & \textbf{73.3} & \underline{57.0} & \underline{33.6} & \underline{58.0} & \underline{63.9} & \underline{70.8} & \textbf{72.5} & \underline{57.6} & \underline{65.1} & \underline{43.8} & \underline{68.1} & \underline{59.2} \\
        \rowcolor{pink!30}~~$\bullet~$REST & \textbf{54.4} & \textbf{56.8} & \textbf{56.2} & \textbf{65.5} & \textbf{58.6} & \underline{72.8} & \textbf{58.6} & \textbf{37.7} & \textbf{59.7} & \textbf{64.1} & \textbf{71.1} & \underline{71.8} & \textbf{61.1} & \textbf{65.6} & \textbf{45.4} & \textbf{68.8} & \textbf{60.5} \\
    \cmidrule{1-18}
	\end{tabular}
	}
	 \end{threeparttable}
	 \end{center}
\end{table*}

\begin{table*}[t]
\centering
\caption{Comparisons with state-of-the-art methods on FIQ-O benchmark under \textbf{\textsc{OQS on both image and text modalities}} regarding Recall@10 metric. $^{*}$ denotes the image corruption and $^{\dagger}$ denotes the text corruption.}
\label{tab: fiq-c}
\resizebox{0.95\linewidth}{!}{
\begin{tabular}{l|ccc|ccc|ccc|ccc|>{\columncolor{blue!8}}c}
\multicolumn{1}{c}{} & \multicolumn{3}{c}{Noise$^{*}$} & \multicolumn{3}{c}{Blur$^{*}$} & \multicolumn{3}{c}{Weather$^{*}$} & \multicolumn{3}{c}{Digital$^{*}$}   \\
 Methods & Char.$^{\dagger}$ & Word$^{\dagger}$ & Sent.$^{\dagger}$ & Char.$^{\dagger}$ & Word$^{\dagger}$ & Sent.$^{\dagger}$ & Char.$^{\dagger}$ & Word$^{\dagger}$ & Sent.$^{\dagger}$ & Char.$^{\dagger}$ & Word$^{\dagger}$ & Sent.$^{\dagger}$ & Avg. \\
\hline
BLIP-2 & 23.7 & 28.1 & 41.3 & 23.1 & 29.1 & 41.7 & 24.3 & 29.8 & 42.9 & 28.9 & 34.2 & 46.7 & 32.8 \\
~~$\bullet~$Tent & 25.1 & 29.2 & 41.4 & 23.5 & 29.7 & 42.1 & 25.7 & 30.5 & 43.3 & 28.7 & 34.8 & 47.3 & 33.4 \\
~~$\bullet~$PL & 24.4 & 28.6 & 41.5 & 23.2 & 29.4 & 42.0 & 25.3 & 30.3 & 43.8 & 28.8 & 34.8 & 47.1 & 33.3 \\
~~$\bullet~$SHOT & 25.1 & 29.2 & 41.4 & 23.5 & 29.7 & 42.1 & 25.6 & 30.5 & 43.3 & 28.7 & 34.7 & 47.3 & 33.4 \\
~~$\bullet~$EATA & 24.4 & 29.1 & 41.5 & 23.6 & 29.5 & 42.3 & 25.2 & 30.5 & 43.2 & 29.1 & 34.7 & 47.5 & 33.4 \\
~~$\bullet~$SAR & 24.0 & 28.7 & 41.3 & 23.7 & 29.6 & 42.1 & 24.8 & 30.1 & 43.1 & 29.1 & 34.5 & 47.1 & 33.2 \\
~~$\bullet~$READ & 24.2 & 28.6 & 41.8 & 23.6 & 29.3 & 42.2 & 25.0 & 30.0 & 43.2 & 28.8 & 34.7 & 47.4 & 33.2 \\
~~$\bullet~$COME & 23.4 & 28.8 & 41.0 & 21.9 & 29.0 & 41.8 & 24.4 & 30.2 & 43.2 & 28.0 & 34.3 & 46.9 & 32.7 \\
~~$\bullet~$TSA & 24.6 & 28.5 & 41.1 & 23.1 & 29.7 & 42.0 & 25.5 & 30.4 & 43.8 & 29.0 & 34.8 & 46.9 & 33.3 \\
\rowcolor{pink!30}~~$\bullet~$TCR & \underline{26.5} & \underline{30.7} & \underline{42.4} & \underline{25.3} & \underline{30.3} & \underline{42.6} & \underline{27.3} & \underline{32.1} & \underline{44.6} & \underline{29.3} & \underline{35.3} & \underline{47.6} & \underline{34.5} \\
\rowcolor{pink!30}
~~$\bullet~$REST & \textbf{27.0} & \textbf{31.2} & \textbf{42.7} & \textbf{25.9} & \textbf{31.2} & \textbf{42.8} & \textbf{27.8} & \textbf{32.5} & \textbf{44.8} & \textbf{29.6} & \textbf{35.6} & \textbf{47.7} & \textbf{34.9} \\
\hline
\end{tabular}
}
\end{table*}

\begin{table}[t]
\centering
\caption{Comparisons with state-of-the-art methods on AudioSet-O benchmark under \textbf{\textsc{OQS on both video or audio modalities}} regarding Recall@1 metric. ``N.'', ``B.'', ``D'', ``W'' denotes ``Noise'', ``Blur'', ``Digital'', ``Weather'' corruption categories, respectively.}
\label{tab: audioset-c}
\Large
\resizebox{0.85\linewidth}{!}{
\begin{tabular}{l|cccc|c|>{\columncolor{blue!8}}c}
\multicolumn{1}{c}{} & \multicolumn{4}{c}{AR} & \multicolumn{1}{c}{VR} \\
 Methods & N.+B. & N.+D. & B.+D. & N.+B.+D. & N.+W. & Avg.  \\
\hline
CAV-MAE & 5.7 & 5.3 & 9.0 & 7.9 & 5.3 & 6.6 \\
~~$\bullet~$Tent & 1.1 & 0.5 & 0.6 & 1.3 & 3.5 & 1.4 \\
~~$\bullet~$PL & 3.4 & 2.5 & 0.6 & 2.9 & 4.3 & 2.7 \\
~~$\bullet~$SHOT & 1.2 & 0.5 & 0.6 & 1.3 & 3.8 & 1.5 \\
~~$\bullet~$EATA & 4.2 & 5.2 & 5.0 & 7.1 & 5.9 & 5.5 \\
~~$\bullet~$SAR & 7.8 & 5.7 & 5.2 & 8.9 & 5.4 & 6.6 \\
~~$\bullet~$READ & 5.2 & 3.9 & 4.5 & 6.7 & 6.3 & 5.3 \\
~~$\bullet~$COME & 2.4 & 0.5 & 1.2 & 1.5 & 3.4 & 1.8 \\
~~$\bullet~$TSA & 6.8 & 4.4 & 4.3 & 5.7 & 5.8 & 5.4 \\
\rowcolor{pink!30}~~$\bullet~$TCR & \underline{10.6} & \underline{9.1} & \underline{10.1} & \underline{11.2} & \underline{8.2} & \underline{9.8} \\
\rowcolor{pink!30}
~~$\bullet~$REST & \textbf{12.6} & \textbf{10.6} & \textbf{11.5} & \textbf{12.9} & \textbf{9.5} & \textbf{11.4} \\
\hline
\end{tabular}
}
\end{table}

\begin{table*}[ht]
\centering
\caption{Comparisons with state-of-the-art methods in benchmarks with \textbf{\textsc{online distribution shift on both query and gallery sets}}. In the table, `ID" and ``OD" refer to ``In-Domain" and ``Out-Domain", respectively. ``@K" represents Recall@K metric for retrieval results.}
\label{tab: qgs}
\Large
\resizebox{0.9\linewidth}{!}{
\begin{tabular}{l|cc|cc|cc|cc|cc|c|c|>{\columncolor{blue!8}}c}
\multicolumn{1}{c}{} & \multicolumn{2}{c}{General2Flickr} & \multicolumn{2}{c}{General2COCO} & \multicolumn{2}{c}{General2Nocaps (ID)} & \multicolumn{2}{c}{General2Nocaps (OD)} & \multicolumn{2}{c}{AudioSet2VGG} & \multicolumn{1}{c}{FIQ2CIRR} & \multicolumn{1}{c}{CIRR2FIQ}   \\
Methods & TR@1 & IR@1 & TR@1 & IR@1 & TR@1 & IR@1 & TR@1 & IR@1 & AR@1 & VR@1 & CIR@1 & CIR@10 & Avg. \\
\hline
Base & 70.0 & 68.3 & 59.3 & 45.4 & 88.2 & 74.9 & 81.9 & 67.8 & 17.8 & 14.6 & 36.5 & 34.6 & 54.9 \\
~~$\bullet~$Tent & 81.9 & 68.5 & 61.7 & 41.7 & 88.5 & 75.4 & 82.7 & 68.9 & 17.4 & 17.8 & 35.9 & 32.5 & 56.1 \\
~~$\bullet~$PL   & 79.9 & 69.2 & 61.8 & 43.6 & 87.8 & 75.5 & 81.4 & 68.8 & 17.7 & 13.7 & \underline{36.6} & 33.5 & 55.8 \\
~~$\bullet~$SHOT & 85.6 & 69.6 & 64.5 & 34.3 & 88.2 & 75.4 & 82.4 & 68.6 & 17.5 & 18.9 & 35.9 & 32.4 & 56.1 \\
~~$\bullet~$EATA & 83.4 & 69.4 & 64.0 & 46.7 & 88.0 & 75.4 & 82.0 & 68.8 & 17.8 & 16.0 & 36.3 & 33.7 & 56.8 \\
~~$\bullet~$SAR  & 81.6 & 68.2 & 63.6 & 46.7 & 88.5 & 75.1 & 81.7 & 67.9 & 17.8 & 14.6 & \underline{36.6} & 33.1 & 56.3 \\
~~$\bullet~$READ & 79.5 & 70.2 & 63.5 & 44.9 & 87.5 & 75.9 & 80.7 & 68.7 & 17.8 & 16.2 & \textbf{36.7} & 34.0 & 56.3 \\
~~$\bullet~$COME & 84.2 & 67.0 & 63.7 & 26.3 & 88.7& 74.9 & 83.1 & 66.3 & 17.5 & 16.5 & 35.9 & 32.0 & 54.7 \\
~~$\bullet~$TSA  & 80.7 & 69.4 & 62.7 & 44.2 & 88.5 & 75.4 & 81.1 & 68.6 & 17.8 & 12.8 & \underline{36.6} & 34.4 & 56.0 \\
\rowcolor{pink!30}~~$\bullet~$TCR  & \underline{86.8} & \underline{70.1} & \underline{68.9} & \underline{48.1} & \underline{89.7} & \underline{76.0} & \underline{87.2} & \underline{68.9} & \underline{18.0} & \underline{19.4} & \textbf{36.7} & \underline{35.6} & \underline{58.8} \\
\rowcolor{pink!30}
~~$\bullet~$REST & \textbf{87.8} & \textbf{70.2} & \textbf{69.4} & \textbf{48.6} & \textbf{90.2} & \textbf{76.2} & \textbf{87.8} & \textbf{69.0} & \textbf{18.3} & \textbf{20.4} & \textbf{36.7} & \textbf{36.1} & \textbf{59.2} \\
\hline
\end{tabular}
}
\end{table*}

\begin{table*}[ht]
\centering
\caption{Comparisons with state-of-the-art methods on benchmarks under \textbf{\textsc{DQS}} setting.}
\label{tab: mixed}
\resizebox{0.8\linewidth}{!}{
\begin{tabular}{l|cc|cc|cc|cc|c|c|>{\columncolor{blue!8}}c}
\multicolumn{1}{c}{} & \multicolumn{2}{c}{Fashion-Gen} & \multicolumn{2}{c}{COCO-D} & \multicolumn{2}{c}{Flickr-D} & \multicolumn{2}{c}{AudioSet-D} & \multicolumn{1}{c}{CIRR-D} & \multicolumn{1}{c}{FIQ-D}  \\
Methods & TR@1 & IR@1 & TR@1 & IR@1 & TR@1 & IR@1 & AR@1 & VR@1 & CIR@1 & CIR@10 & Avg.  \\
\hline
Base        & 19.9 & 26.1 & 45.0 & 40.3 & 61.7 & 61.5 & 9.8 & 6.9 & \underline{34.1} & 31.7 & 33.7 \\
~~$\bullet~$Tent & 11.7 & 24.6 & 12.8 & 5.7  & 42.3 & 48.1 & 2.6 & 5.2 & 34.0           & 31.1 & 21.8 \\
~~$\bullet~$PL   & 13.1 & 23.7 & 19.8 & 7.6  & 47.7 & 60.7 & 7.5 & 6.6 & 33.8           & 31.4 & 25.2 \\
~~$\bullet~$SHOT & 9.8  & 23.0 & 14.4 & 5.7  & 40.2 & 49.0 & 3.0 & 2.7 & 34.0           & 31.1 & 21.3 \\
~~$\bullet~$EATA & 17.0 & 25.2 & 38.2 & 39.7 & 51.4 & 62.0 & 5.2 & 7.8 & 34.0           & 31.4 & 31.2 \\
~~$\bullet~$SAR  & 14.8 & 26.2 & 40.2 & 40.3 & 53.2 & 61.4 & 9.8 & 6.9 & 34.0           & 31.6 & 31.8 \\
~~$\bullet~$READ & 10.6 & 25.3 & 32.6 & 37.4 & 55.7 & 61.6 & 7.1 & 8.6 & 33.9           & 31.6 & 30.4 \\
~~$\bullet~$COME & 2.3  & 2.5  & 9.9  & 4.7  & 47.0 & 36.5 & 3.0 & 7.0 & 33.4           & 31.4 & 17.8 \\
~~$\bullet~$TSA  & 13.4 & 25.6 & 30.4 & 35.9 & 60.0 & 60.5 & 7.8 & 7.2 & 34.0           & 31.6 & 30.6 \\
\rowcolor{pink!30}~~$\bullet~$TCR  & \underline{21.0} & \underline{29.3} & \underline{48.6} & \underline{40.7} & \underline{67.1} & \underline{62.1} & \underline{10.9} & \underline{9.1} & 34.0 & \underline{32.1} & \underline{35.4} \\
\rowcolor{pink!30}
~~$\bullet~$REST & \textbf{23.4} & \textbf{29.9} & \textbf{51.8} & \textbf{41.1} & \textbf{70.3} & \textbf{62.9} & \textbf{12.3} & \textbf{10.0} & \textbf{34.2} & \textbf{32.4} & \textbf{36.8} \\
\hline
\end{tabular}
}
\end{table*}

\subsection{Comparisons under OQS setting}
In this section, we compare REST with nine state-of-the-art test-time domain adaptation methods, including online TTA methods~\cite{Tent,pl,shot}, robust TTA methods~\cite{come,tcr}, multi-modal TTA methods~\cite{READ,tsa} and TTA-oriented continual learning methods~\cite{EATA,SAR}.
Thanks to the formulated query prediction in Eq.~\ref{eq: prediction}, we could benchmark the aforementioned recognition-oriented methods for CMR with QS.
For fair comparisons, we select the optimal temperature (Eq.~\ref{eq: prediction}) for all the TTA baselines upon each dataset according to Fig.~\ref{fig: ablation}(a). 

For a quantified study on DQS, we first inject one synthetic corruption into the queries from the source domain, thereby obtaining the target-domain data.
Accordingly, the pre-trained model is fine-tuned on the source-domain data, and then conducts online adaptation to the corruption-injected target-domain data.
Note that, for the evaluation in video-audio and composed image retrieval, we select the most representative subtype in each corruption category and inject it into the queries, \textit{e.g.}, ``Gaussian noise'' for ``General Noise'' in image modality, ``OCR noise'' for ``Character-level noise'' in text modality, and ``Windy noise'' for ``Weather noise'' in audio modality.
See more details in Supplementary Material~\textsc{III}-A.
The results on the corruption-injected benchmarks across three CMR tasks are summarized in Tables~\ref{tab: coco-o-image}-\ref{tab: fiq-c}, where one could have the following observations and conclusions.
On the one hand, due to the inability to manipulate both query uniformity and the query-gallery alignment, most existing TTA methods achieve marginal improvements or even degrade the retrieval performance. 
In contrast, REST could restore the well-established common space inherited from the source model, thus significantly outperforming all the baselines across various pre-trained models and CMR tasks;
On the other hand, REST demonstrates greater robustness against more severe distribution shift such as ``Zoom'' in Table~\ref{tab: coco-o-image}, ``Noise + Blur'' in Table~\ref{tab: audioset-c}, and ``Blur + Char.'' in Table~\ref{tab: fiq-c}, whereas most baseline methods suffer significant performance degradation under these challenging visual/textual or composite distribution shifts. 

To further verify the effectiveness of REST against OQS, we conduct the experiments on benchmarks with real-world distribution shift.
Specifically, in the ``General2$\mathcal{D}_{T}$'' setting, we conduct online adaptation on ${D}_{T}$ using the pre-trained model, while in the ``$\mathcal{D}_{S}$2$\mathcal{D}_{T}$'' setting, the CMR model fine-tuned on $\mathcal{D}_{S}$ is employed to adapt to ${D}_{T}$ in an online manner.
From the results in Table~\ref{tab: qgs}, REST could achieve stable performance improvements, especially in the following cases:
i) as the size of the gallery set increases, existing TTA methods suffer from increasing performance degradation, \textit{e.g.}, Tent yields 7.8\% improvement in ``General2Flickr'' but drops to -1.3\% in ``General2COCO''. 
Such a phenomenon supports the claim in Section~\ref{sec: problem formulation} that the excessively large gallery size would hinder TTA from accommodating well for CMR tasks;
ii) as the out-of-domain challenge becomes more severe (from ``General2Nocaps (ID)'' to ``General2Nocaps (OD)''), most existing TTA methods fail to deliver performance gains.

\begin{figure*}[ht]
\centering
\includegraphics[width=0.95\linewidth]{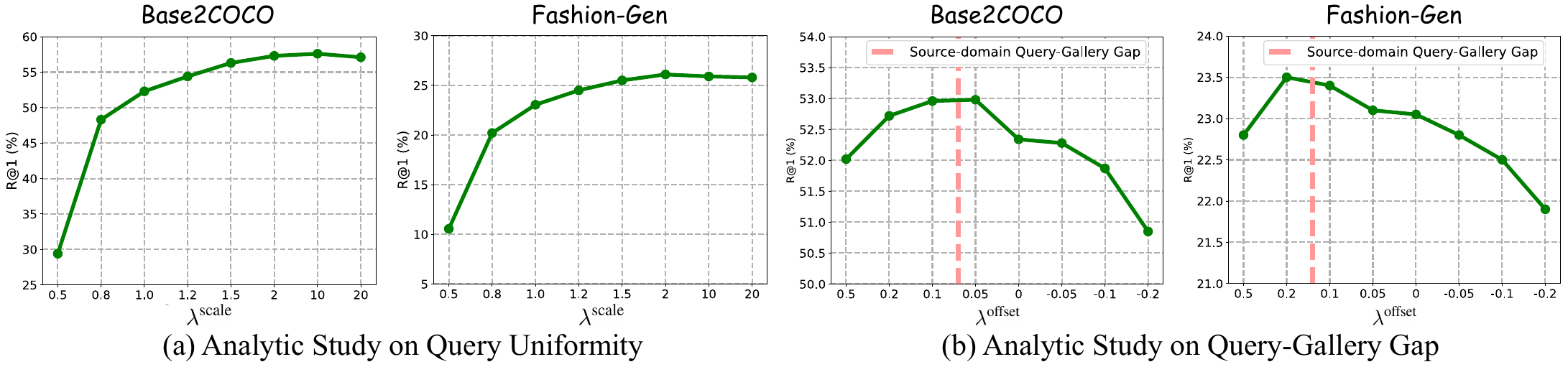}
\caption{
Observation of the query uniformity and query-gallery gap. 
Increasing $\lambda^{\operatorname{scale}}$ indicates enlarging query uniformity, while the decreasing $\lambda^{\operatorname{offset}}$ indicates narrowing query-gallery gap.
Notably, $\lambda^{\operatorname{scale}}=1.0$ and $\lambda^{\operatorname{offset}}=0$ represent no scaling and no offset, respectively. 
}
\label{fig: observation}
\end{figure*}

\begin{figure*}[ht]
\centering
\includegraphics[width=0.95\linewidth]{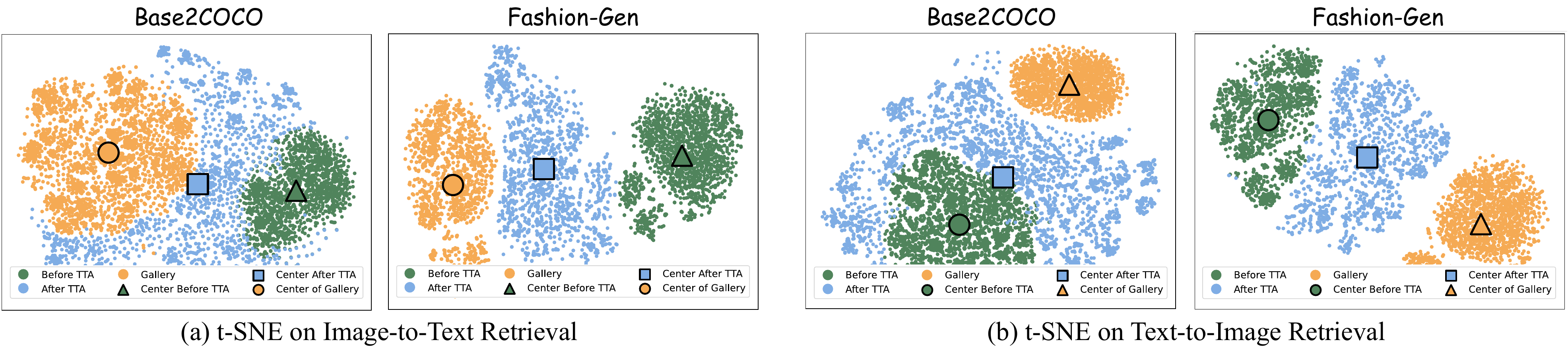}
\caption{
t-SNE visualization of query and candidate embeddings after employing the proposed REST.
}
\label{fig: tsne}
\end{figure*}

\subsection{Comparisons under DQS setting}
In this section, we conduct experiments to investigate the effectiveness of REST against DQS.
In the experiments, we conduct evaluations on the real-world Fashion-Gen benchmark and five corruption-injected benchmarks, including COCO-D, Flickr-D, AudioSet-D, CIRR-D, and FIQ-D, for evaluation.
Specifically, the Fashion-Gen benchmark inherently contains data from 48 product domains, \textit{e.g.}, SHIRTS and SKIRTS.
For the corruption-injected benchmarks, we randomly select one corruption or one combination of corruptions from all possible distribution shift types and then inject it into each query from the source domain, thus obtaining the target-domain data.
As a result, the queries in the data stream always originate from diverse domains with different corruptions.
Accordingly, the pre-trained model is fine-tuned on the source-domain data, and then adapts to target-domain data with diverse shift.
More specifically, for video–audio retrieval, each temporal-sequence query is divided into several segments, with one selected corruption injected into each segment.
For composed image retrieval, each multi-modal query is perturbed by one randomly selected image corruption and one randomly chosen text corruption.
Note that the gradient decoupling module is employed under the DQS setting to prevent the retrieval model from forgetting the general knowledge.

As illustrated in Table~\ref{tab: mixed}, existing methods suffer from severe performance degradation, even though most of them could achieve improvements in the OQS setting, which demonstrates that diverse query shift is more challenging.
Although some TTA-oriented continual learning methods are proposed, such as SAR and EATA, they strive for invariant parameters and thus hinder the acquisition of new knowledge, resulting in suboptimal CMR performance.
In contrast, REST employs a more plausible paradigm that self-adaptively removes gradients conflicting with general knowledge, thereby achieving consistent performance improvements.

\subsection{Parameter Analysis and Ablation Studies}
In this section, we carry out a series of parameter analysis
and ablation studies to investigate the effectiveness of REST.
Unless otherwise stated, all the experiments are conducted under the ``General2COCO'' and ``Fashion-Gen'' settings using the BLIP ViT-B/16 model.

\subsubsection{Ablation Studies of Loss Terms}
To verify the importance of each loss term, we investigate the variants of the QS-robust objective function in Table~\ref{tab: ablation_loss}, where one could have the following conclusions.
On the one hand, the proposed query uniformity loss, query-gallery gap loss, and query-gallery consistency loss could work independently and accordingly improve performance.
On the other hand, since these three losses focus on achieving robustness against query shift from different perspectives, employing all of them together would lead to the optimal CMR performance.

\begin{figure*}[t]
\centering
\includegraphics[width=0.9\linewidth]{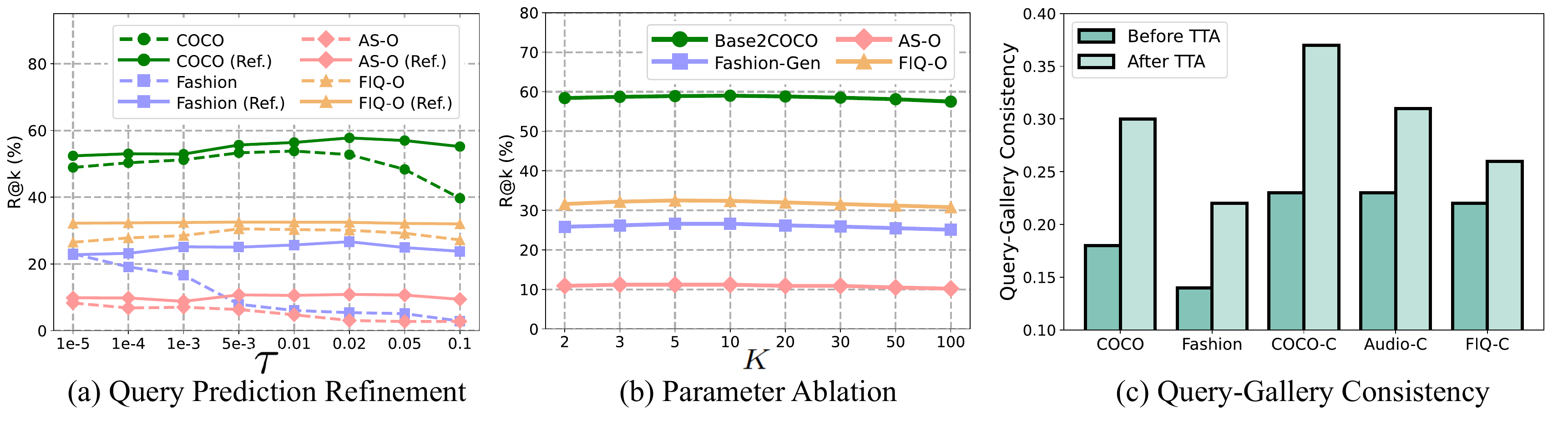}
\caption{
Finer-grained Ablation studies. 
(a) Parameter analysis of $\tau$ on the vanilla TTA method Tent w/ (solid line) and w/o (dotted line) the query prediction refinement module.  
(b) Parameter analysis of the selected neighbors $K$ in Eq.~\ref{eq: new prediction}.
(c) Analysis of query-gallery consistency before and after employing REST.
}
\label{fig: ablation}
\end{figure*}

\begin{table}[t]
\centering
\caption{Ablation study of the loss terms under the ``General2COCO'' setting, where $\checkmark$ denotes the module is adopted.}
\resizebox{0.75\linewidth}{!}{
\begin{tabular}{ccccccc}
$\mathcal{L}_{U}$ & $\mathcal{L}_{C}$ & $\mathcal{L}_{G}$ & TR@1 & IR@1 & Avg. \\
\midrule
   &    &    &   59.3 & 45.4 & 52.4 \\
$\checkmark$ &    &    & 65.9 & 46.5 & 56.2 \\
   & $\checkmark$ &    & 67.2 & 47.2 & 57.2\\
   &    & $\checkmark$ & 61.9 & 46.6 & 54.3 \\
$\checkmark$ & $\checkmark$ &    & 68.2 & 47.7 & 58.0 \\
$\checkmark$ &    & $\checkmark$ & 66.7 & 47.6 & 57.2 \\
   & $\checkmark$ & $\checkmark$ & 67.8 & 47.8 & 57.8 \\
$\checkmark$ & $\checkmark$ & $\checkmark$ & 69.4 & 48.6 & 59.0 \\
\bottomrule
\end{tabular}}
\label{tab: ablation_loss}
\end{table}

\subsubsection{Analytic Study on Query Uniformity and Query-Gallery Gap}
As discussed in Introduction, the query shift would diminish the uniformity of queries and amplify the query-gallery gap.
For an in-depth understanding, we conduct analytic experiments to investigate how the two characteristics affect the retrieval performance.
Specifically, to investigate the influence of query uniformity, we manually scale the distribution of queries by moving the embeddings in common space as follows,
\begin{equation}
    (z_i^Q)^{\text {scale }}=\overline{\mathbf{Z}}^{Q}+\lambda^{\text{scale}}\left(z_i^Q-\overline{\mathbf{Z}}^{Q}\right),
    \label{eq: scale}
\end{equation}
where $\overline{\mathbf{Z}}^{Q}=\frac{1}{N^{Q}}\sum_{i=1}^{N^{Q}}\mathbf{z}_i^Q$ denotes the center of queries, $\lambda^{\text{scale}}$ indicates the scaling factor.
From the results in Fig.~\ref{fig: observation} (a), enlarging the query uniformity would boost the retrieval performance, but not vice versa. 
Such a phenomenon supports the claim that higher uniformity would guarantee the discrimination between queries and thus improve the performance.
To investigate the influence of the query-gallery gap, we manually shift the embeddings of all queries towards closing the query-gallery gap as follows,
\begin{equation}
    (z_i^Q)^{\text {offset}}=z_i^Q-\lambda^{\text {offset}} \left(\overline{\mathbf{Z}}^{Q}-\overline{\mathbf{Z}}^{G}\right),
    \label{eq: offset}
\end{equation}
where $\overline{\mathbf{Z}}^{G}=\frac{1}{N^{G}}\sum_{i=1}^{N^{G}}\mathbf{z}_i^G$ denotes the center of the candidates, $\lambda^{\text{offset}}$ indicates the offset factor.
As illustrated in Fig.~\ref{fig: observation}(b), one could observe that monotonously narrowing the query-gallery gap would not always bring the performance gains, while the estimated source-domain query-gallery gap $\Delta_{S}$ might be a plausible criterion for the rectification.
Note that the embeddings are all $\ell 2$-normalized after scaling or shifting.

To further qualitatively study the effectiveness of REST, we carry out the t-SNE visualization on both the query and gallery sets before and after the TTA process. 
As shown in Fig.~\ref{fig: tsne}, we could observe that REST could enhance the scatter of queries and eliminate the discrepancy between the query and gallery centers.
In other words, REST achieves robustness against query shift by enlarging query uniformity and narrowing the query-gallery gap.

\begin{table}[t]
\centering
\caption{Ablation Study of the gradient decoupling (GD) module. ``w/o'' and ``w'' denote without and with, respectively. The unit of ``Angle'' is degrees.}
\label{tab: gra_effect}
\Large
\resizebox{0.9\linewidth}{!}{
\begin{tabular}{lcccccc}
 & \multicolumn{3}{c}{Fashion} & \multicolumn{3}{c}{COCO-D} \\
 Methods & Angle & TR@1 & IR@1 & Angle & TR@1 & IR@1 \\
\midrule
Tent w/o GD & 120.1 & 11.7 & 24.6 & 131.4 & 12.8 & 5.7 \\
Tent w/ GD  & 86.1  & 20.8 & 27.5 & 87.9  & 46.5 & 40.4 \\
Ours w/o GD & 98.6 & 22.3 & 28.6 & 101.4 & 50.4 & 40.6 \\
Ours w/ GD  & 83.2  & 23.4 & 29.9 & 80.0  & 51.8 & 41.1 \\
\bottomrule
\end{tabular}
}
\end{table}

\subsubsection{Analytic Study on Query Prediction Refinement}
The proposed query prediction refinement module aims to prevent the model from either underfitting or overfitting by excluding irrelevant samples in the gallery set.
To verify this claim, we conduct analytic experiments under varying temperatures in Eq.~\ref{eq: prediction} and Eq.~\ref{eq: new prediction}.
As shown in Fig.~\ref{fig: ablation}(a), one could have the following conclusions:
i) selecting an appropriate temperature for TTA across different datasets is highly challenging;
ii) although a very low temperature (e.g., $1e^{-4}$) yields better performance for some datasets, the performance degrades since the model tends to overfit noisy query predictions;
iii) query prediction refinement module not only stabilizes the temperature selection for all the datasets, but also alleviates both the underfitting and overfitting problems by excluding irrelevant candidates.
In particular, REST demonstrates stable performance improvements within the range $[0.001, 0.05]$ and achieves the best performance at $\tau=0.02$.

To construct diverse negatives for Eq.~\ref{eq: loss_rem}, we employ the Top-$K$ selection on the negatives within the mini-batch.
Here, we conduct experiments to explore the impact of different values of $K$. 
The results in Fig.~\ref{fig: ablation}(b) demonstrate that an appropriate choice of $K$ (e.g., $K=10$) would enhance the diversity of negatives and thus boost retrieval performance.

\subsubsection{Analytic Study on Query-Gallery Consistency}
Beyond the efforts on manipulating query uniformity and query-gallery gap, REST could enhance the consistency between query and gallery sets, thus facilitating the associations of correct query-candidate pairs.
To validate this claim, we conduct an analytical experiment on query-gallery consistency before and after employing REST.
The results in Fig.~\ref{fig: ablation}(c) demonstrate that employing REST significantly improves the consistency of query-candidate pairs and thus boosts the retrieval performance.

\subsubsection{Analytic Study on Gradient Decoupling}
As discussed in Section~\ref{sec: Gradient Decoupling}, we propose the gradient decoupling module to eliminate the gradient conflict with the general knowledge.
Here, we conduct the experiments to investigate the effectiveness of the designed GD module.
From the results in Table~\ref{tab: gra_effect}, one could have the following conclusions.
On the one hand, the refined gradient $\boldsymbol{\hat{G}}$ would not conflict with general knowledge and thus support knowledge accumulation for better retrieval performance, which further verifies the correctness of Theorem~\ref{theorem: gradient}.
On the other hand, since the proposed GD module is method-agnostic, integrating the GD module with other TTA methods could also alleviate the general knowledge forgetting problem, further validating its effectiveness.

\section{Conclusion}
In this paper, we formally study the query shift problem in cross-modal retrieval.
By delving into the underlying negative impacts of the query shift problem, we reveal that query shift will not only disrupt the well-constructed common space inherited from the source model, but also steer the model toward acquiring domain-specific knowledge that conflicts with general knowledge.
To address the query shift problem, we develop a new test-time adaptation method, dubbed REST.
In brief, REST formulates the predictions of queries and then performs QS-robust objectives on these query predictions to preserve the well-constructed common space.
After that, REST adopts a gradient decoupling module to prevent the CMR model from forgetting the general knowledge, thus achieving harmonious adaptation.
For comprehensive evaluations, we provide 20 benchmarks featuring either synthetic corruptions or real-world distribution shifts.
In the future, we plan to explore more applications that might suffer from the query-shift problem, such as image captioning and visual question answering.

\bibliographystyle{IEEEtran}
\bibliography{reference}

\vspace{0.2in}

\renewcommand{\thefigure}{A.\arabic{figure}}
\setcounter{figure}{0}
\renewcommand{\thetable}{A.\arabic{table}}
\setcounter{table}{0}
\renewcommand{\thesection}{\Alph{section}}
\setcounter{section}{0}

\begin{leftline}
	{
		\Large{\textsc{Supplementary Materials}}
	}
\end{leftline}
\startcontents[sections]
{
    \hypersetup{linkcolor=black}
    \printcontents[sections]{l}{1}{}
}

\section{Definitions}
In this section, we provide formal definitions of the aforementioned query uniformity, query-gallery gap, and query-gallery consistency in the manuscript, which are essential characteristics in multimodal representation learning~\cite{multi_modal_deviation,liu1,liu2,yin2}. 
In the following, we elaborate on each of them one by one.

\textbf{Query Uniformity: } the uniformity of the given queries is defined as
\begin{equation}
    \text{Uniformity}=\frac{1}{N^{Q}}\sum_{i=1}^{N^Q}\|z_{i}^{Q}-\overline{\mathbf{Z}}^{Q}\|,
    \label{eq: def uni}
\end{equation}
where $\overline{\mathbf{Z}}^{Q}=\frac{1}{N^{Q}}\sum_{i=1}^{N^Q}z_{i}^{Q}$ denotes the center of the queries.
Intuitively, a low query uniformity indicates that the queries are overly compact in the common space, making it hard for the retrieval model to distinguish between them.

\textbf{Query-Gallery Gap: } the gap between query and gallery sets is defined as
\begin{equation}
    \text{Gap}=\|\overline{\mathbf{Z}}^{Q}-\overline{\mathbf{Z}}^{G}\|,
    \label{eq: def: gap}
\end{equation}
where $\overline{\mathbf{Z}}^{G}=\frac{1}{N^{G}}\sum_{i=1}^{N^G}z_{i}^{G}$ denotes the center of the given candidates.
As illustrated in Fig. 3, either an excessively small or large query-gallery gap would disrupt the well-constructed common space, thereby degrading the retrieval performance. 
It is worth noting that the gap between modalities has been proved to be an inherent characteristic of multimodal pre-trained models in recent work ~\cite{MindGap}.
Different from prior studies on modality gaps, we further investigate the properties of the query-gallery gap, where the given queries may be either unimodal or multimodal.

\textbf{Query-Gallery Consistency: } the consistency between query and gallery sets is defined as
\begin{equation}
    \text{Consistency}=\frac{1}{\sum_{i,j}\mathcal{C}(z_i^Q,z_j^G)}\sum_{i,j}\mathcal{C}(z_i^Q,z_j^G) \operatorname{cos} (z_i^Q,z_j^G),
    \label{eq: def: consistency}
\end{equation}
where $\operatorname{cos}(\cdot,\cdot)$ denotes the cosine similarity, and $\mathcal{C}(z_i^Q,z_j^G)=1$ $i.f.f$ $(z_i^Q, z_j^G)$ denotes correctly associated pair, and $0$ otherwise.
For the retrieval task, it is desirable that queries and their corresponding candidates are highly similar in the common space, thereby the associated samples could be easily identified.

\section{Theoretical Analysis}
In this section, we provide detailed proofs of Theorem 1
and Theorem 2 in the manuscript.
\subsection{Proofs of Theorem 1}
\renewcommand{\thetheorem}{1}
\begin{theorem}
    Let $\hat{p}$ denotes the refined query prediction of the given query $x$, where $\hat{p}_{i}$ denotes the probability that $x$ is associated with the $i$-th candidate. For the entropy minimization objective $\mathcal{L}_{EM}$, the probability gradient of the easy negative $\hat{p}_{m}$ exceeds that of the hard negative $\hat{p}_{n}$ whenever $\hat{p}_{m}< \hat{p}_{n} \leq \frac{1}{e}$. 
    Mathematically, we have
    \begin{equation}
        \left|\frac{\partial \mathcal{L}_{EM}}{\partial \hat{p}_{m}}\right| > \left|\frac{\partial \mathcal{L}_{EM}}{\partial \hat{p}_{n}}\right|
    \end{equation}
    where $|\cdot|$ denotes the absolute value.
\end{theorem}

\begin{proof}
According to the definition of Shannon entropy~\cite{s_entropy}, \textit{i.e.}, $\mathcal{L}_{EM}=-\sum_{i}\hat{p}_i\log\hat{p}_i$, the first-order derivative of $\mathcal{L}_{EM}$ with respect to $\hat{p}_i$ is 
\begin{equation}
    \frac{\partial \mathcal{L}_{\mathrm{EM}}}{\partial \hat{p}_{i}}
    = -\bigl(\log \hat{p}_{i} + 1\bigr).
    \label{eq: theorem_1_1}
\end{equation}
Next, the second-order derivative could be derived as
\begin{equation}
    \frac{\partial^2 \mathcal{L}_{\mathrm{EM}}}{\partial \hat{p}_{i}^2}
    = -\frac{1}{\hat{p}_{i}} < 0,
\end{equation}
which implies that $\frac{\partial \mathcal{L}_{\mathrm{EM}}}{\partial \hat{p}_{i}}$ monotonically decreasing.

In this case, according to Eq.~\ref{eq: theorem_1_1}, the first-order derivative is positive when $\hat{p}_{i} < 1/e$.
Consequently, for all $\hat{p}_{i}\in(0,1/e]$, we have 
\begin{equation}
    \left|\frac{\partial \mathcal{L}_{\mathrm{EM}}}{\partial \hat{p}_{i}}\right|
    = \frac{\partial \mathcal{L}_{\mathrm{EM}}}{\partial \hat{p}_{i}}.
\end{equation}
As the first-order derivative is strictly decreasing, for any $\hat{p}_{m}< \hat{p}_{n} \le \tfrac{1}{e}$, it follows that
\begin{equation}
\left|\frac{\partial \mathcal{L}_{\mathrm{EM}}}{\partial \hat{p}_{m}}\right|
>
\left|\frac{\partial \mathcal{L}_{\mathrm{EM}}}{\partial \hat{p}_{n}}\right|,
\end{equation}
which proves the theorem.
\end{proof}

\subsection{Proofs of Theorem 2}
\renewcommand{\thetheorem}{2}
\begin{theorem}
The refined gradient $\boldsymbol{\hat{G}}$ never conflicts with the retrieval-specific direction $\boldsymbol{G}_r$, \textit{i.e.}, $\boldsymbol{\hat{G}}\cdot \boldsymbol{G}_{r} \geq 0$
\end{theorem}

\begin{proof}
In the following, we elaborate on the two possible cases for the angle between $G_d$ and $G_r$.

\textit{Case 1:} When $\boldsymbol{G}_d \cdot \boldsymbol{G}_r \geq 0$, we have $\boldsymbol{G}_{\|} \cdot \boldsymbol{G}_r\geq 0$.
Accordingly, we could derive $\boldsymbol{\hat{G}} \cdot \boldsymbol{G}_r $ as follows,
\begin{equation}
    \begin{aligned}
        \boldsymbol{\hat{G}} \cdot \boldsymbol{G}_r &= (W_d \ \boldsymbol{G}_{\perp} + W_d \ \boldsymbol{G}_{\|})\cdot \boldsymbol{G}_r\\
        &=W_d \ \boldsymbol{G}_{\perp} \cdot \boldsymbol{G}_r + W_d \ \boldsymbol{G}_{\|} \cdot \boldsymbol{G}_r\\
        &=W_d \ \boldsymbol{G}_{\|} \cdot \boldsymbol{G}_r \geq 0.
    \end{aligned}
\end{equation}
where $\boldsymbol{G}_{\perp}$ is orthogonal to $\boldsymbol{G}_r$ so that $\ \boldsymbol{G}_{\perp} \cdot \boldsymbol{G}_r=0$.

\textit{Case 2:} When $\boldsymbol{G}_d \cdot \boldsymbol{G}_r < 0$, we have $\boldsymbol{\hat{G}} = W_d \ \boldsymbol{G}_{\perp}$. In this case, we could derive $\boldsymbol{\hat{G}} \cdot \boldsymbol{G}_r $ as follows,
\begin{equation}
    \boldsymbol{\hat{G}} \cdot \boldsymbol{G}_r = W_d \ \boldsymbol{G}_{\perp} \cdot \boldsymbol{G}_r = 0.
\end{equation}

As a result, we could conclude that $\boldsymbol{\hat{G}} \cdot \boldsymbol{G}_r \geq 0$, which demonstrates that the refined gradient $\boldsymbol{\hat{G}}$ would never conflict with the retrieval-specific direction $\boldsymbol{G}_r$.
\end{proof}

\begin{table*}[t]
\centering
\caption{Corruptions in different modalities. The \textbf{bold} entries indicate the most representative subtypes within each corruption category.}
\label{tab: corruptions}
\resizebox{0.85\linewidth}{!}{
\begin{tabular}{lp{11cm}}
\toprule
\textbf{Modality} & \textbf{Category: Subtypes} \\
\midrule
Image &
i) Noise: \textbf{Gaussian noise}, Impulse noise, Shot noise, Speckle noise; \newline
ii) Blur: Defocus blur, Glass blur, \textbf{Motion blur}, Zoom blur; \newline
iii) Weather: Brightness, Frost, \textbf{Fog}, Snow; \newline
iv) Digital: Contrast, Elastic, \textbf{JPEG compression}, Pixelate. \\
\midrule
Text &
i) Character-level: \textbf{OCR}, Character Delete (CD), Character Insert (CI), Character Replace (CR), Character Swap (CS); \newline
ii) Word-level: Insert Punctuation (IP), Word Deletion (WD), Word Insertion (WR), Word Swap (WS), \textbf{Synonym Replacement (SR)}; \newline
iii) Sentence-level: Active, Casual, Formal, Passive, \textbf{Backtranslation}. \\
\midrule
Video &
i) Noise: \textbf{Gaussian noise}, Impulse noise, Shot noise, Speckle noise; \newline
ii) Blur: Defocus blur, Glass blur, \textbf{Motion blur}, Zoom blur; \newline
iii) Digital: Contrast, Elastic, \textbf{JPEG compression}, Pixelate. \\
\midrule
Audio &
i) Noise: Gaussian noise, Crowd noise, \textbf{Traffic noise}; \newline
ii) Weather: Rainy noise, Thunder noise, \textbf{Windy noise}; \\
\bottomrule
\end{tabular}
}
\end{table*}

\vspace{-0.1in}

\begin{table*}[t]
\centering
\caption{Statistics of the used datasets in this work.}
\label{tab: benchmarks}
\resizebox{0.8\linewidth}{!}{
\begin{tabular}{lcccc}
\toprule
\textbf{Dataset} & \textbf{Type} & \textbf{Scale} & \textbf{Relation} \\
\midrule
COCO~\cite{COCO} & Image-Text & 5,000 images / 25,000 captions & 1-to-5 \\
Flickr~\cite{Flickr} & Image-Text & 1,000 images / 5,000 captions & 1-to-5 \\
Nocaps~\cite{Nocaps} & Image-Text & 1,562 images / 15,620 captions & 1-to-10 \\
Fashion-Gen~\cite{Fashion-Gen} & Image-Text & 32,528 images / 32,528 captions & 1-to-1 \\
RSICD~\cite{RSICD} & Image-Text & 1,000 images / 5,000 captions & 1-to-5 \\
RSITMD~\cite{RSITMD} & Image-Text & 450 images / 2,250 captions & 1-to-5 \\
AudioSet~\cite{audioset} & Video-Audio & 3,000 videos / 3,000 audios & 1-to-1 \\
VGGSound~\cite{vggsound} & Video-Audio & 3,000 videos / 3,000 audios & 1-to-1 \\
FIQ~\cite{FashionIQ} & Composed Image Retrieval & 6,016 references \& modifications / images & (1+1)-to-1 \\
CIRR~\cite{cirr} & Composed Image Retrieval & 4,148 references \& modifications / images & (1+1)-to-1 \\
\bottomrule
\end{tabular}
}
\vspace{-0.1in}
\end{table*}

\section{More Implementation Details}
In this section, we will provide more details about the benchmarks and the experimental setting.

\subsection{More Details about the Benchmarks}
In the manuscript, we conduct the experiments under various settings, \textit{i.e.}, online query shift and diverse query shift.
Here, we provide more details about the benchmarks used in the two different settings. 
Specifically, we first introduce the datasets employed in Table~\ref{tab: benchmarks} and then we elaborate on the constructed benchmarks as follows.

\begin{table*}[htbp]
    \vspace{-0.2in}
    \caption{Comparisons with state-of-the-art methods on COCO-O benchmark under \textbf{\textsc{OQS on the image modality}} regarding the Recall@1 metric. The best results are marked in \textbf{bold} and the second best results are \underline{underlined}.}
    \label{tab: coco-o-image-large}
 \begin{center}
 \vspace{-0.1in}
 \begin{threeparttable}
    \Large
    \resizebox{0.98\linewidth}{!}{
 	\begin{tabular}{l|cccc|cccc|cccc|cccc|>{\columncolor{blue!8}}c}
 	\multicolumn{1}{c}{} & \multicolumn{4}{c}{Noise} & \multicolumn{4}{c}{Blur} & \multicolumn{4}{c}{Weather} & \multicolumn{4}{c}{Digital}  \\
 	 Methods & Gauss. & Shot & Impul. & Speckle & Defoc. & Glass & Motion & Zoom & Snow & Frost & Fog & Brit. & Contr. & Elastic & Pixel & JPEG & Avg.  \\
    \cmidrule{1-18}
        BLIP ViT-L/16 & 50.3 & 51.8 & 51.1 & 61.6 & 53.7 & 72.1 & 49.4 & 14.5 & 44.0 & 57.5 & 61.8 & 70.5 & 37.3 & 50.6 & 32.0 & 70.5 & 51.8 \\ 
        ~~$\bullet~$Tent & 46.3 & 49.3 & 46.7 & 58.4 & 52.2 & 71.8 & 47.5 & 12.3 & 41.9 & 56.2 & 60.9 & 69.7 & 35.7 & 48.3 & 29.4 & 69.6 & 49.8 \\ 
        ~~$\bullet~$PL & 35.4 & 41.8 & 42.8 & 62.0 & 34.7 & 73.2 & 23.3 & 3.5 & 40.8 & 56.3 & 57.4 & 71.5 & 17.4 & 56.1 & 31.6 & 70.2 & 44.9 \\
        ~~$\bullet~$SHOT & 33.5 & 37.5 & 47.7 & 64.6 & 24.9 & 73.7 & 17.6 & 1.5 & 37.5 & 56.7 & 55.1 & \underline{72.7} & 25.5 & 53.1 & 11.3 & \underline{70.7} & 42.7 \\
        ~~$\bullet~$EATA & 46.2 & 53.5 & 49.5 & 63.8 & 56.5 & 73.8 & 52.6 & 18.4 & 50.6 & 59.1 & 64.5 & 72.1 & 40.7 & 55.4 & 43.5 & \underline{70.7} & 54.4  \\
        ~~$\bullet~$SAR & 45.9 & 50.2 & 47.3 & 63.1 & 51.1 & 73.8 & 47.2 & 11.6 & 40.8 & 58.9 & 60.7 & 71.6 & 33.6 & 54.0 & 34.4 & 70.5 & 50.9  \\
        ~~$\bullet~$READ & 38.1 & 48.0 & 43.3 & 63.5 & 43.6 & 73.4 & 43.6 & 22.0 & 44.5 & 56.5 & 62.2 & 71.9 & 32.9 & 49.6 & 27.5 & 70.6 & 49.5  \\
        ~~$\bullet~$COME & 23.8 & 21.2 & 22.7 & 53.3 & 14.3 & 73.2 & 11.3 & 1.7 & 26.6 & 50.9 & 46.6 & 72.2 & 21.3 & 51.1 & 10.2 & 69.9 & 35.7 \\
        ~~$\bullet~$TSA & 41.0 & 47.7 & 46.2 & 61.0 & 49.1 & 73.9 & 42.7 & 22.1 & 40.6 & 58.5 & 60.3 & 71.5 & 32.0 & 53.7 & 32.2 & 70.6 & 50.2 \\
        \rowcolor{pink!30}~~$\bullet~$TCR & \textbf{58.2} & \underline{60.5} & \textbf{58.9} & \underline{66.6} & \underline{60.0} & \underline{74.5} & \underline{61.3} & \underline{39.3} & \underline{58.8} & \underline{65.2} & \underline{71.7} & 72.6 & \underline{56.6} & \underline{68.0} & \underline{48.7} & 70.3 & \underline{62.0} \\
        \rowcolor{pink!30}~~$\bullet~$REST & \underline{57.5} & \textbf{61.4} & \underline{58.7} & \textbf{67.0} & \textbf{62.2} & \textbf{74.9} & \textbf{62.1} & \textbf{40.5} & \textbf{62.5} & \textbf{65.4} & \textbf{71.9} & \textbf{73.1} & \textbf{57.0} & \textbf{69.0} & \textbf{54.9} & \textbf{70.8} & \textbf{63.1} \\
    \cmidrule{1-18}
	\end{tabular}
	}
	 \end{threeparttable}
	 \end{center}
\end{table*}

\begin{table*}[htbp]
    \vspace{-0.2in}
    \caption{Comparisons with state-of-the-art methods on COCO-O benchmark under \textbf{\textsc{OQS on the text modality}} regarding the Recall@1 metric. The best results are marked in \textbf{bold} and the second best results are \underline{underlined}.}
    \label{tab: coco-o-text-large}
\newcommand{\tabincell}[2]{\begin{tabular}{@{}#1@{}}#2\end{tabular}}
\vspace{-0.1in}
 \begin{center}
 \begin{threeparttable}
    \Large
    \resizebox{0.9\linewidth}{!}{
 	\begin{tabular}{l|ccccc|ccccc|ccccc|>{\columncolor{blue!8}}c}
 	\multicolumn{1}{c}{} & \multicolumn{5}{c}{Character-level} & \multicolumn{5}{c}{Word-level} & \multicolumn{5}{c}{Sentence-level}  \\
 	Methods & OCR & CI & CR & CS & CD & SR & RI & RS & RD & IP & Formal & Casual & Passive & Active & Backtrans & Avg.  \\
    \cmidrule{1-17}
        BLIP ViT-B/16 & 31.4 & 11.3 & 9.4 & 18.9 & 11.4 & 43.6 & 51.5 & 50.3 & 50.6 & 56.8 & 56.6 & 56.2 & 54.9 & 56.8 & 54.2 & 40.9 \\ 
        ~~$\bullet~$Tent & 23.0 & 2.4 & 2.3 & 7.0 & 2.1 & 43.9 & 52.5 & 51.2 & 50.5 & 56.6 & 57.0 & 56.6 & 55.6 & 57.3 & 54.4 & 38.2 \\ 
        ~~$\bullet~$PL & 33.0 & 3.6 & 4.5 & 18.1 & 7.9 & 44.9 & 53.3 & 51.8 & 51.3 & \underline{57.0} & 57.0 & 56.6 & 55.8 & 57.3 & 54.5 & 40.5 \\ 
        ~~$\bullet~$SHOT & 22.9 & 2.4 & 2.4 & 6.5 & 2.1 & 43.9 & 52.6 & 51.1 & 50.5 & 56.6 & 56.9 & 56.6 & 55.6 & 57.3 & 54.4 & 38.1 \\ 
        ~~$\bullet~$EATA & 33.0 & 11.7 & 10.3 & 18.4 & 11.8 & 44.9 & 53.2 & 51.7 & \underline{51.2} & 56.9 & 57.0 & 56.7 & \underline{56.0} & \underline{57.4} & 54.5 & 41.7 \\ 
        ~~$\bullet~$SAR & 31.8 & 11.5 & 9.8 & 18.5 & 11.6 & 43.7 & 51.5 & 50.3 & 50.6 & 56.8 & 56.5 & 56.2 & 55.0 & 56.8 & 54.2 & 41.0 \\ 
        ~~$\bullet~$READ & 31.7 & 10.8 & 9.6 & 18.1 & 11.1 & 44.1 & 53.0 & 51.5 & 51.0 & 56.9 & \underline{57.1} & 56.7 & \underline{56.0} & 57.1 & 54.5 & 41.3 \\ 
        ~~$\bullet~$COME & 18.8 & 1.5 & 1.4 & 5.6 & 2.4 & 43.6 & 51.5 & 50.7 & 50.6 & 56.6 & 57.0 & 56.7 & 55.5 & \underline{57.4} & \underline{54.6} & 37.6 \\ 
        ~~$\bullet~$TSA & 30.9 & 10.2 & 9.1 & 16.8 & 11.1 & 43.9 & 51.9 & 51.0 & 50.9 & \underline{57.0} & 57.0 & \underline{56.8} & 55.7 & 57.3 & 54.5 & 40.9 \\ 
        \rowcolor{pink!30}~~$\bullet~$TCR & \underline{34.0} & \underline{13.8} & \underline{11.8} & \underline{19.5} & \underline{13.1} & \underline{45.1} & \underline{53.4} & \underline{52.0} & \underline{51.2} & \underline{57.0} & \underline{57.1} & \underline{56.8} & 55.8 & \underline{57.4} & \underline{54.6} & \underline{42.2} \\ 
        \rowcolor{pink!30}
        ~~$\bullet~$Ours & \textbf{34.5} & \textbf{14.1} & \textbf{12.1} & \textbf{19.8} & \textbf{13.2} & \textbf{45.5} & \textbf{53.9} & \textbf{52.3} & \textbf{51.4} & \textbf{57.3} & \textbf{57.4} & \textbf{57.0} & \textbf{56.2} & \textbf{57.6} & \textbf{54.7} & \textbf{42.5} \\
    \cmidrule{1-17}
        BLIP ViT-L/16 & 34.5 & 12.3 & 11.1 & 19.7 & 12.9 & 46.0 & 54.4 & 54.0 & 53.5 & 59.3 & 59.1 & 58.8 & 57.8 & 59.4 & 56.7 & 43.3 \\ 
        ~~$\bullet~$Tent & 34.0 & 12.3 & 11.0 & 19.6 & 12.9 & 46.5 & 54.2 & 53.8 & 53.4 & 59.3 & 59.1 & 58.8 & 57.6 & 58.9 & 56.5 & 43.2 \\ 
        ~~$\bullet~$PL & 34.6 & 3.3 & 2.3 & 14.9 & 4.9 & 47.0 & 55.0 & 54.5 & 52.9 & \textbf{59.4} & \textbf{59.5} & 58.9 & 58.1 & \underline{59.6} & 56.8 & 41.4 \\
        ~~$\bullet~$SHOT & 33.8 & 1.5 & 1.2 & 5.4 & 2.6 & 46.4 & 54.5 & 53.6 & 52.3 & 58.2 & 58.9 & 58.7 & 57.8 & 59.2 & 56.4 & 40.0 \\
        ~~$\bullet~$EATA & 35.6 & 13.3 & 11.3 & 20.3 & 13.2 & \underline{47.2} & 55.4 & 54.2 & \underline{53.6} & 59.2 & 59.1 & \underline{59.0} & 57.9 & 59.4 & 56.8 & 43.7 \\
        ~~$\bullet~$SAR & 34.5 & 13.1 & 11.2 & 20.3 & 13.1 & 46.7 & 54.4 & 54.0 & 53.5 & \textbf{59.4} & 59.1 & 58.8 & 57.8 & 59.4 & 56.7 & 43.5 \\
        ~~$\bullet~$READ & 35.3 & 12.2 & 10.9 & 19.1 & 12.7 & 47.3 & 55.1 & \underline{54.6} & 53.3 & \textbf{59.4} & 59.3 & \textbf{59.1} &58.1 & 59.6 & 56.7 & 43.5 \\
        ~~$\bullet~$COME & 33.3 & 2.7 & 2.2 & 6.8 & 3.6 & 46.4 & 54.6 & 53.3 & 51.9 & 57.9 & 58.9 & 58.7 & 57.7 & 59.2 & 56.3 & 40.2 \\
        ~~$\bullet~$TSA & 33.9 & 11.7 & 10.2 & 17.8 & 11.9 & 46.6 & 54.1 & 53.5 & 52.4 & 59.2 & 59.0 & 58.9 & 58.1 & 59.5 & 56.6 & 42.9 \\
        \rowcolor{pink!30}~~$\bullet~$TCR & \underline{36.5} & \underline{14.6} & \underline{13.2} & \underline{21.2} & \underline{14.3} & \textbf{47.8} & \underline{56.1} & 54.4 & \textbf{53.7} & \textbf{59.4} & \underline{59.4} & \underline{59.0} & \underline{58.2} & \underline{59.6} & \underline{56.9} & \underline{44.3} \\
        \rowcolor{pink!30}
        ~~$\bullet~$Ours & \textbf{36.6} & \textbf{14.8} & \textbf{13.3} & \textbf{21.3} & \textbf{14.5} & \textbf{47.8} & \textbf{56.5} & \textbf{55.2} & \underline{53.6} & \underline{59.3} & \textbf{59.5} & \textbf{59.1} & \textbf{58.3} & \textbf{59.7} & \textbf{57.0} & \textbf{44.4} \\
    \cmidrule{1-17}
	\end{tabular}
	}
	 \end{threeparttable}
	 \end{center}
\vspace{-0.3in}
\end{table*}

\begin{table}[htbp]
\centering
\caption{Comparisons with state-of-the-art methods on the RSICD and RSITMD benchmarks with \textbf{\textsc{online distribution shift on both query and gallery sets}}.}
\label{tab: remote sensing}
\vspace{-0.1in}
\Large
\resizebox{0.75\linewidth}{!}{
\begin{tabular}{l|cc|cc|>{\columncolor{blue!8}}c}
\multicolumn{1}{c}{} & \multicolumn{2}{c}{GeneralRSICD} & \multicolumn{2}{c}{GeneralRSITMD}   \\
Methods & TR@1 & IR@1 & TR@1 & IR@1 & Avg. \\
\hline
Base  & 6.4 & 6.8 & 7.6 & 10.4 & 7.8 \\
~~$\bullet~$TENT & 5.6 & 6.3 & 6.7 & 9.6  & 7.0 \\
~~$\bullet~$PL   & 5.0 & 6.3 & 6.7 & 10.1 & 7.0 \\
~~$\bullet~$SHOT & 5.5 & 6.3 & 6.7 & 9.7  & 7.0 \\
~~$\bullet~$EATA & 6.3 & 6.4 & 8.0 & 9.9  & 7.7 \\
~~$\bullet~$SAR  & 6.3 & 6.1 & 7.9 & 9.9  & 7.6 \\
~~$\bullet~$READ & 3.7 & 6.0 & 7.6 & 10.4 & 6.9 \\
~~$\bullet~$COME & 4.0 & 6.2 & 7.0 & 9.2  & 6.6 \\
~~$\bullet~$TSA  & 5.8 & 6.9 & 8.0 & 10.4 & 7.8 \\
\rowcolor{pink!30}~~$\bullet~$TCR  & 8.0 & 7.0 & 9.1 & 10.6 & 8.7 \\
\rowcolor{pink!30}
~~$\bullet~$Ours & \textbf{8.4} & \textbf{7.1} & \textbf{9.8} & \textbf{10.8} & \textbf{9.0} \\
\hline
\end{tabular}
}
\vspace{-0.1in}
\end{table}

\textbf{Online Query Shift:} 
we construct five benchmarks based on the five widely used datasets and accordingly obtain Flickr-O, COCO-O, AudioSet-O, FIQ-O, and CIRR-O benchmarks.

As shown in Table~\ref{tab: corruptions}, for the image-text retrieval, we introduce 16 and 15 types of corruption to the image and text modality, respectively.
For the video-audio retrieval, we introduce representative corruptions from three categories (\textit{i.e.}, Noise, Blur, and Digital) for the video modality and two categories (\textit{i.e.}, Noise and Weather) for the audio modality.
Note that for multi-distribution shifts, \textit{e.g.}, ``Noise+Blur'' or ``Noise+Blur+Digital'', the temporal query is split into $k$ segments if $k$ corruption types are used, and each segment is injected by a different corruption.
For composed image retrieval, we introduce representative corruptions from four categories for the image modality (\textit{i.e.}, Noise, Blur, Weather, and Digital) and three categories for the text modality (\textit{i.e.}, character-, word-, and sentence-level).

Moreover, we employ real-world benchmarks with distribution shifts on both query and gallery sets, including COCO, Flickr, Nocaps, VGGSound, FIQ, and CIRR dataset. 
Specifically, we adopt the following two methods for evaluation on real-world benchmarks.
General$\mathcal{D}_{T}$: the pre-trained model (\textit{e.g.}, BLIP) directly adapt to the target-domain data $\mathcal{D}_{T}$.
$\mathcal{D}_{S}$2$\mathcal{D}_{T}$: the model is first fine-tuned on the source-domain data $\mathcal{D}_{S}$ and then adapt to the target-domain data $\mathcal{D}_{T}$.

\textbf{Diverse Query Shift:} we employ the Fashion-Gen benchmark and five corruption-injected benchmarks to simulate queries originating from diverse domains.
Specifically, the real-world Fashion-Gen benchmark covers 48 user demands, \textit{e.g.}, tops, pants, and sneakers, which could be naturally treated as diverse domains.
Besides, we construct five corruption-injected benchmarks based on Flickr, COCO, AudioSet, FIQ, and CIRR datasets.
For the corruption-injected benchmarks, we randomly select a corruption or a combination of corruptions from all possible distribution shift types and inject it into each query from the source domain, thus obtaining target domain data with the same size as $\mathcal{D}_{S}$. As a result, each query might encounter different corruptions, \textit{i.e.}, diverse distribution shift, thereby obtaining Flickr-D, COCO-D, AudioSet-D, FIQ-D, and CIRR-D benchmarks.

\subsection{More Details about the Experimental Setting}
To guarantee the performance of the baselines, we carefully design the temperature in Eq.~1 for each dataset and adopt the optimal value for the TTA baselines. According to Fig.~5(a), the temperature is fixed at 0.001 for Fashion-Gen and 0.01 for all other benchmarks. 
It is worth noting that Tent is employed to mine the optimal temperature in Fig.~5(a), as most existing TTA methods are variants of Tent.

\section{More Experiments}
In this section, we provide additional experimental results on the larger pre-trained model, more benchmarks, single-stream retrieval model, and efficiency analysis.

\begin{table*}[!t]
    \caption{
    Comparisons with state-of-the-art methods on Flickr-O benchmark under \textbf{\textsc{OQS on the image modality}} regarding the Recall@1 metric. 
    }
    \label{tab: flickr-c-image}
 \begin{center}
 \begin{threeparttable}
    \large
    \resizebox{0.98\linewidth}{!}{
 	\begin{tabular}{l|cccc|cccc|cccc|cccc|>{\columncolor{blue!8}}c}
 	\multicolumn{1}{c}{} & \multicolumn{4}{c}{Noise} & \multicolumn{4}{c}{Blur} & \multicolumn{4}{c}{Weather} & \multicolumn{4}{c}{Digital}  \\
 	Query Shift & Gauss. & Shot & Impul. & Speckle & Defoc. & Glass & Motion & Zoom & Snow & Frost & Fog & Brit. & Contr. & Elastic & Pixel & JPEG & Avg.  \\
    \cmidrule{1-18}
        BLIP ViT-B/16 
        & 49.8 & 56.6 & 50.3 & 71.6 & 53.1 & 84.5 & 47.4 & 15.5 & 66.4 & 80.4 & 79.5 & 85.5 & 60.6 & 53.3 & 35.1 & 80.3 & 60.6\\
        ~~$\bullet~$Tent 
        & 54.9 & 54.9 & 54.3 & 73.1 & 53.3 & 85.3 & 47.9 & 1.6 & 67.2 & 80.9 & 79.6 & 86.8 & 63.6 & 53.4 & 35.4 & 81.4 & 60.9\\
        ~~$\bullet~$PL 
        & 51.2 & 60.3 & 49.2 & 73.7 & 29.9 & 86.2 & 28.9 & 2.3 & 69.7 & 83.1 & 81.9 & 87.4 & 66.1 & 57.8 & 38.4 & 81.4 & 59.2\\
        ~~$\bullet~$SHOT 
        & 54.0 & 63.3 & 56.2 & 76.6 & 39.0 & 87.8 & 21.7 & 1.5 & 72.6 & 84.9 & 84.0 & 88.9 & 70.6 & 63.1 & 32.9 & 82.6 & 61.2\\
        ~~$\bullet~$EATA 
        & 55.5 & 60.5 & 55.8 & 75.8 & 64.6 & 86.2 & 52.2 & 8.5 & 72.0 & 83.7 & 82.5 & 87.9 & 68.4 & 60.1 & 45.9 & 81.6 & 65.1\\
        ~~$\bullet~$SAR 
        & 54.8 & 62.5 & 55.6 & 75.2 & 48.3 & 87.2 & 34.8 & 15.5 & 71.9 & 83.1 & 82.2 & 87.9 & 68.2 & 60.3 & 42.2 & 81.4 & 63.2\\
        ~~$\bullet~$READ 
        & 50.1 & 58.2 & 52.2 & 74.8 & 63.7 & 87.0 & 55.1 & 2.2 & 71.7 & 83.8 & 81.9 & 87.7 & 67.4 & 62.3 & 42.5 & 81.4 & 63.9\\
        ~~$\bullet~$COME 
        & 55.1 & 62.1 & 50.2 & 76.5 & 40.7 & 87.5 & 26.0 & 1.6 & 70.7 & 84.2 & 83.5 & 88.3 & 71.4 & 61.2 & 33.6 & 81.8 & 60.9\\
        ~~$\bullet~$TSA 
        & 54.7 & 59.0 & 53.3 & 73.7 & 62.7 & 86.7 & 55.3 & 4.4 & 72.1 & 83.2 & 82.2 & 87.3 & 65.6 & 58.9 & 47.9 & 82.0 & 64.3\\
        \rowcolor{pink!30}~~$\bullet~$TCR 
        & \underline{62.0} & \underline{66.6} & \underline{61.4} & \underline{80.0} & \underline{68.1} & \underline{87.9} & \underline{65.2} & \underline{39.9} & \underline{78.2} & \underline{85.2} & \underline{85.7} & \underline{89.5} & \underline{75.1} & \underline{73.1} & \underline{56.8} & \underline{83.3} & \underline{72.4}\\
        \rowcolor{pink!30}~~$\bullet~$REST 
        & \textbf{64.9} & \textbf{70.0} & \textbf{66.3} & \textbf{82.4} & \textbf{72.2} & \textbf{88.3} & \textbf{70.2} & \textbf{47.3} & \textbf{79.7} & \textbf{85.8} & \textbf{86.2} & \textbf{89.9} & \textbf{77.7} & \textbf{78.3} & \textbf{70.2} & \textbf{84.3} & \textbf{75.9}\\
    \cmidrule{1-18}
        BLIP ViT-L/16 
        & 58.2 & 61.0 & 59.7 & 76.9 & 66.4 & 88.5 & 62.5 & 33.4 & 67.7 & 81.5 & 79.3 & 89.1 & 60.4 & 66.4 & 46.5 & 85.0 & 67.7\\
        ~~$\bullet~$Tent 
        & 61.3 & 64.3 & 63.3 & 77.6 & 70.8 & 88.7 & 62.8 & 31.5 & 70.4 & 83.8 & 81.1 & 89.2 & 61.2 & 68.7 & 52.0 & 84.5 & 69.5\\
        ~~$\bullet~$PL 
        & 61.0 & 60.2 & 60.3 & 78.4 & 67.8 & 89.8 & 63.3 & 23.0 & 70.0 & 84.0 & 81.1 & 90.1 & 61.0 & 68.8 & 52.4 & 84.2 & 68.5\\
        ~~$\bullet~$SHOT 
        & 62.1 & 65.5 & 66.7 & 79.3 & 71.8 & 90.6 & 64.9 & 18.6 & 73.4 & 84.9 & 82.5 & 90.5 & 64.1 & 72.0 & 57.2 & 86.4 & 70.7\\
        ~~$\bullet~$EATA 
        & 62.0 & 65.1 & 64.5 & 78.9 & 70.2 & 89.5 & 63.3 & 33.1 & 71.9 & 83.7 & 81.2 & 89.3 & 61.6 & 69.3 & 53.0 & 85.8 & 70.2\\
        ~~$\bullet~$SAR 
        & 61.1 & 64.4 & 63.7 & 79.7 & 71.6 & 90.3 & 64.4 & 27.6 & 70.6 & 83.4 & 81.0 & 89.7 & 62.4 & 70.1 & 53.3 & 85.3 & 69.9\\
        ~~$\bullet~$READ 
        & 61.5 & 61.0 & 62.1 & 78.3 & 69.6 & 89.5 & 62.5 & 37.2 & 72.1 & 83.6 & 81.4 & 89.9 & 61.3 & 67.6 & 52.5 & \underline{86.8} & 69.8\\
        ~~$\bullet~$COME 
        & 65.6 & 67.2 & 68.6 & 79.7 & 72.6 & 90.4 & 66.4 & 18.8 & 71.8 & 84.4 & 81.7 & 91.0 & 65.7 & 72.1 & 56.9 & 86.2 & 71.2\\
        ~~$\bullet~$TSA 
        & 60.7 & 61.6 & 61.0 & 78.4 & 69.0 & 89.1 & 64.9 & 36.6 & 71.7 & 84.0 & 80.6 & 89.4 & 61.2 & 69.5 & 48.7 & 86.4 & 69.6\\
        \rowcolor{pink!30}~~$\bullet~$TCR 
        & \underline{69.0} & \underline{73.3} & \underline{70.8} & \underline{83.8} & \underline{74.7} & \underline{92.1} & \underline{75.2} & \underline{54.4} & \underline{79.6} & \underline{87.7} & \textbf{86.4} & \underline{91.6} & \underline{72.7} & \underline{82.8} & \underline{67.7} & 86.4 & \underline{78.0}\\
        \rowcolor{pink!30}~~$\bullet~$REST 
        & \textbf{71.6} & \textbf{74.5} & \textbf{73.0} & \textbf{84.6} & \textbf{76.5} & \textbf{92.3} & \textbf{75.6} & \textbf{54.5} & \textbf{80.9} & \textbf{88.2} & \underline{86.3} & \textbf{92.5} & \textbf{74.9} & \textbf{84.3} & \textbf{68.7} & \textbf{88.5} & \textbf{79.2}\\
    \cmidrule{1-18}
	\end{tabular}
	}
	 \end{threeparttable}
	 \end{center}
\end{table*}

\vspace{-0.2in}

\begin{table*}[!t]
    \caption{Comparisons with state-of-the-art methods on Flickr-O benchmark under \textbf{\textsc{OQS on the text modality}} regarding the Recall@1 metric.}
    \label{tab: flickr-c-text}
 \begin{center}
 \begin{threeparttable}
    \large
    \resizebox{0.9\linewidth}{!}{
 	\begin{tabular}{l|ccccc|ccccc|ccccc|>{\columncolor{blue!8}}c}
 	\multicolumn{1}{c}{} & \multicolumn{5}{c}{Character-level} & \multicolumn{5}{c}{Word-level} & \multicolumn{5}{c}{Sentence-level}  \\
 	 Query Shift & OCR & CI & CR & CS & CD & SR & RI & RS & RD & IP & Formal & Casual & Passive & Active & Backtrans & Avg.  \\
    \cmidrule{1-17}
        BLIP ViT-B/16 & 53.5 & 18.4 & 18.0 & 30.4 & 22.5 & 68.3 & 77.9 & 76.9 & 77.9 & 82.1 & 82.1 & 81.9 & 79.9 & 82.2 & 79.8 & 62.1 \\
        ~~$\bullet~$Tent & 55.4 & 18.6 & 18.2 & 31.1 & 23.0 & 69.6 & 78.8 & 77.7 & 77.9 & 82.2 & 81.9 & 81.8 & 79.6 & 82.0 & 79.9 & 62.5 \\
        ~~$\bullet~$PL & 56.6 & 19.5 & 13.6 & 31.9 & 23.7 & 69.7 & \underline{79.1} & 77.9 & 77.8 & \textbf{82.5} & \underline{82.3} & \underline{82.0} & 80.5 & 82.4 & 79.9 & 62.7 \\
        ~~$\bullet~$SHOT & 56.8 & 14.4 & 10.7 & 31.2 & 21.4 & 69.7 & 79.0 & 77.9 & 77.9 & 82.3 & 82.2 & \underline{82.0} & 80.6 & \textbf{82.6} & \underline{80.0} & 61.9 \\
        ~~$\bullet~$EATA & 55.7 & 19.9 & 19.9 & 31.6 & 23.6 & 69.5 & 78.6 & 77.5 & 77.9 & \underline{82.4} & \underline{82.3} & 81.8 & 80.5 & \underline{82.5} & \textbf{80.1} & 62.9 \\
        ~~$\bullet~$SAR & 53.5 & 20.1 & 19.1 & 32.1 & 23.8 & 68.3 & 77.9 & 76.9 & 77.7 & 82.1 & 82.1 & 81.9 & 79.9 & 82.2 & 79.8 & 62.5 \\
        ~~$\bullet~$READ & 55.8 & 19.7 & 20.6 & 32.0 & 23.5 & 69.3 & 78.6 & 77.6 & 77.8 & \underline{82.4} & 82.2 & 81.8 & 80.5 & \textbf{82.6} & \textbf{80.1} & 63.0 \\
        ~~$\bullet~$COME & 55.1 & 11.8 & 10.7 & 31.1 & 22.4 & 69.4 & 78.2 & \underline{77.9} & 77.8 & 82.2 & 82.1 & \underline{82.0} & 80.6 & 82.4 & \textbf{80.1} & 61.6 \\
        ~~$\bullet~$TSA & 55.4 & 19.8 & 20.2 & 31.6 & 23.2 & 69.2 & 78.3 & 77.3 & 77.7 & \underline{82.4} & 82.0 & 81.9 & 80.2 & 82.4 & 79.9 & 62.8 \\
        \rowcolor{pink!30}~~$\bullet~$TCR & \underline{57.1} & \underline{21.4} & \underline{22.5} & \underline{33.6} & \underline{25.1} & \underline{69.8} & \textbf{79.2} & \textbf{78.0} & \underline{78.0} & \textbf{82.5} & \textbf{82.4} & \textbf{82.1} & \textbf{80.8} & \underline{82.5} & \textbf{80.1} & \underline{63.7} \\
        \rowcolor{pink!30}~~$\bullet~$REST & \textbf{57.7} & \textbf{21.8} & \textbf{22.9} & \textbf{34.1} & \textbf{25.2} & \textbf{69.9} & \textbf{79.2} & \textbf{78.0} & \textbf{78.1} & \textbf{82.5} & \underline{82.3} & \textbf{82.1} & \underline{80.7} & \textbf{82.6} & \textbf{80.1} & \textbf{63.8} \\
    \cmidrule{1-17}
        BLIP ViT-L/16 & 58.0 & 22.2 & 22.0 & 34.1 & 25.1 & 71.2 & 79.9 & 78.9 & 78.8 & 83.3 & 83.1 & 82.7 & 81.7 & \textbf{83.5} & 80.7 & 64.4 \\
        ~~$\bullet~$Tent & 59.0 & 22.4 & 22.1 & 34.5 & 25.3 & 71.4 & 80.3 & \underline{79.3} & 78.8 & \underline{83.4} & 82.8 & 82.7 & 81.8 & 83.3 & 80.7 & 64.5 \\
        ~~$\bullet~$PL & 58.6 & 23.0 & 22.1 & 34.1 & 24.6 & 71.5 & 80.3 & 79.2 & 78.7 & 83.3 & 83.1 & 82.9 & 81.6 & 83.4 & 80.6 & 64.5 \\
        ~~$\bullet~$SHOT & 58.9 & 21.5 & 19.2 & 34.6 & 24.2 & \textbf{71.7} & 80.3 & 79.2 & 78.8 & 83.3 & 83.2 & 83.2 & 81.8 & \textbf{83.5} & 80.6 & 64.3 \\
        ~~$\bullet~$EATA & 59.1 & 23.0 & 23.2 & 35.1 & 25.6 & \textbf{71.7} & 80.3 & \underline{79.3} & 78.8 & \textbf{83.5} & 83.0 & 83.2 & 81.8 & \textbf{83.5} & 80.7 & \underline{64.8} \\
        ~~$\bullet~$SAR & 58.1 & 23.1 & 23.0 & 34.5 & 25.8 & 71.2 & 79.9 & 78.9 & 78.8 & 83.3 & 83.1 & 82.7 & 81.7 & 83.4 & 80.7 & 64.6 \\
        ~~$\bullet~$READ & 58.9 & \underline{23.4} & 23.3 & 34.9 & 25.9 & 71.5 & 80.4 & 79.2 & 78.8 & \textbf{83.5} & \textbf{83.2} & 83.1 & 81.8 & 83.4 & \textbf{80.8} & \underline{64.8} \\
        ~~$\bullet~$COME & 58.6 & 20.8 & 19.4 & 33.8 & 24.7 & \underline{71.6} & 80.2 & 79.2 & 78.7 & 83.2 & 83.0 & 83.0 & 81.6 & 83.3 & 80.5 & 64.1 \\
        ~~$\bullet~$TSA & 58.8 & 22.8 & 22.6 & 34.1 & 24.8 & 71.4 & \underline{80.5} & \underline{79.3} & 78.8 & \textbf{83.5} & 83.1 & 82.8 & 81.8 & \textbf{83.5} & 80.7 & 64.6 \\
        \rowcolor{pink!30}~~$\bullet~$TCR & \underline{59.5} & \textbf{24.4} & \underline{24.7} & \underline{36.2} & \underline{26.9} & \textbf{71.7} & \textbf{80.6} & \textbf{79.4} & \textbf{78.9} & 83.3 & \textbf{83.2} & \textbf{83.4} & \underline{81.9} & \textbf{83.5} & \textbf{80.8} & \underline{65.2} \\
        \rowcolor{pink!30}~~$\bullet~$REST & \textbf{59.6} & \textbf{24.4} & \textbf{25.0} & \textbf{36.4} & \textbf{27.2} & \textbf{71.7} & \underline{80.5} & \textbf{79.4} & \textbf{78.9} & \underline{83.4} & \textbf{83.2} & \underline{83.3} & \textbf{82.0} & \textbf{83.5} & \textbf{80.8} & \textbf{65.3} \\
    \cmidrule{1-17}
	\end{tabular}
	}
	 \end{threeparttable}
	 \end{center}
\end{table*}

\subsection{Experiments with More Pre-trained Model Size}
In the manuscript, we have carried out experiments under ``image-to-text'' retrieval setting on the COCO-O benchmark using the BLIP ViT-B/16 backbone.
Here, we further provide results under ``text-to-image '' retrieval setting and with different pre-trained model sizes to verify the effectiveness of the proposed REST.
The results in Tables~\ref{tab: coco-o-image-large}–\ref{tab: coco-o-text-large} demonstrate that REST consistently achieves substantial improvements over all baselines across various pre-trained model sizes and different CMR settings.

\subsection{Results on the Remote Sensing Benchmarks}
To further verify the effectiveness of the proposed REST, we conduct additional experiments in the remote sensing domain. Specifically, we adopt the pre-trained BLIP as the source model and perform zero-shot retrieval on two remote sensing datasets: RSICD~\cite{RSICD} and RSITMD~\cite{RSITMD}. 
As shown in Table~\ref{tab: remote sensing}, REST consistently achieves the best performance in the challenging remote sensing domain.

\subsection{Results on the Flickr-O Benchmark}
In the manuscript and Tables~\ref{tab: coco-o-image-large}-\ref{tab: coco-o-text-large}, we have reported experimental results on the COCO-O benchmark.
Here, we further present experimental results on the Flickr-O benchmark.
Specifically, we first fine-tune the pre-trained BLIP on the Flickr benchmark, and then perform TTA on the Flickr-O benchmark.
As shown in Tables~\ref{tab: flickr-c-image}–\ref{tab: flickr-c-text}, REST significantly outperforms all baselines across different pre-trained model sizes on the Flickr-O benchmark.

\subsection{Results on the CIRR-O Benchmark}
In the manuscript and Tables~\textsc{III}, we have reported experimental results on the FIQ-O benchmark.
In this section, we conduct more experiments on the CIRR-O Benchmark.
Specifically, we first fine-tune the pre-trained BLIP-2 on the CIRR benchmark, and then perform TTA on the CIRR-O benchmark.
From the results in Table~\ref{tab: cirr-o}, one can observe that most TTA methods fail to deliver performance improvements, whereas the proposed REST achieves stable performance gains.

\begin{table*}[htbp]
\centering
\caption{Comparisons with state-of-the-art methods on CIRR-O benchmark under \textbf{\textsc{OQS on the image and text modalities}} regarding the Recall@1 metric. $^{*}$ denotes the image corruption and $^{\dagger}$ denotes the text corruption.}
\label{tab: cirr-o}
\resizebox{0.9\linewidth}{!}{
\begin{tabular}{l|ccc|ccc|ccc|ccc|>{\columncolor{blue!8}}c}
\multicolumn{1}{c}{} & \multicolumn{3}{c}{Noise$^{*}$} & \multicolumn{3}{c}{Blur$^{*}$} & \multicolumn{3}{c}{Weather$^{*}$} & \multicolumn{3}{c}{Digital$^{*}$}   \\
 Methods & Char.$^{\dagger}$ & Word$^{\dagger}$ & Sent.$^{\dagger}$ & Char.$^{\dagger}$ & Word$^{\dagger}$ & Sent.$^{\dagger}$ & Char.$^{\dagger}$ & Word$^{\dagger}$ & Sent.$^{\dagger}$ & Char.$^{\dagger}$ & Word$^{\dagger}$ & Sent.$^{\dagger}$ & Avg. \\
\hline
Base & 30.2 & 32.4 & 44.7 &29.3 & 32.9 & 43.4 & 30.3 & 33.1 & 45.2 & 32.6 & 35.1 & 47.2 & 36.4 \\
~~$\bullet~$Tent & 30.4 & \underline{32.9} & \underline{44.9} & 29.2 & 33.0 & \underline{43.4} & \underline{30.4} & \underline{33.4} & 45.2 & 32.2 & 35.2 & 47.0 & 36.4 \\
~~$\bullet~$PL & 30.4 & 32.7 & \underline{44.9} & 29.2 & 32.7 & \underline{43.4} & \underline{30.4} & 33.2 & 45.1 & 32.5 & 35.2 & 47.1 & 36.4 \\
~~$\bullet~$SHOT & 30.4 & \underline{32.9} & 44.8 & \underline{29.3} & 32.9 & \textbf{43.5} & \underline{30.4} & 33.2 & 45.1 & 32.3 & 35.3 & 47.0 & 36.4 \\
~~$\bullet~$EATA & 30.6 & \underline{32.9} & 44.7 & \textbf{29.4} & 33.1 & 43.2 & \underline{30.4} & 33.1 & 45.3 & 32.5 & 35.3 & \underline{47.3} & 36.5 \\
~~$\bullet~$SAR & 30.3 & 32.4 & 44.7 & \underline{29.3} & 32.9 & \underline{43.4} & 30.3 & 33.1 & 45.3 & 32.5 & 35.1 & 47.2 & 36.4 \\
~~$\bullet~$READ & 30.2 & 32.5 & \underline{44.9} & 29.1 & 32.7 & 43.3 & 30.2 & 33.0 & 45.0 & 32.5 & 35.0 & 46.9 & 36.3 \\
~~$\bullet~$COME & 30.0 & 32.7 & 44.3 & 27.6 & 32.3 & 42.3 & 29.1 & 33.0 & 44.1 & 32.1 & \underline{35.4} & \underline{47.3} & 35.9 \\
~~$\bullet~$TSA & 30.4 & 32.5 & 44.8 & 28.8 & 32.7 & 43.3 & 30.3 & 33.1 & 45.2 & \underline{32.6} & 35.0 & 46.9 & 36.3 \\
\rowcolor{pink!30}~~$\bullet~$TCR & \underline{31.2} & 32.8 & \textbf{45.0} & 29.1 & \underline{33.3} & \textbf{43.5} & \textbf{31.0} & \underline{33.4} & \underline{45.4} & \textbf{32.7} & \underline{35.4} & \underline{47.3} & \underline{36.7} \\
\rowcolor{pink!30}
~~$\bullet~$Ours & \textbf{31.4} & \textbf{33.0} & \textbf{45.0} & \underline{29.3} & \textbf{33.8} & \underline{43.4} & \textbf{31.0} & \textbf{33.5} & \textbf{45.7} & \textbf{32.7} & \textbf{35.5} & \textbf{47.4} & \textbf{36.8} \\
\hline
\end{tabular}
}
\end{table*}

\begin{table*}[ht]
    \caption{Comparisons with state-of-the-art methods on COCO-O benchmark under \textbf{\textsc{OQS on the image modality}} regarding the Recall@1 metric. The best results are marked in \textbf{bold} and the second best results are \underline{underlined}.}
    \label{tab: coco-o-image-single}
 \begin{center}
 \begin{threeparttable}
    \Large
    \resizebox{0.98\linewidth}{!}{
 	\begin{tabular}{l|cccc|cccc|cccc|cccc|>{\columncolor{blue!8}}c}
 	\multicolumn{1}{c}{} & \multicolumn{4}{c}{Noise} & \multicolumn{4}{c}{Blur} & \multicolumn{4}{c}{Weather} & \multicolumn{4}{c}{Digital}  \\
 	 Methods & Gauss. & Shot & Impul. & Speckle & Defoc. & Glass & Motion & Zoom & Snow & Frost & Fog & Brit. & Contr. & Elastic & Pixel & JPEG & Avg.  \\
    \cmidrule{1-18}
        BLIP ViT-B/16 & 51.0 & 52.5 & 50.4 & 64.7 & 51.6 & 75.7 & 48.2 & 11.5 & 43.3 & \underline{61.1} & 65.0 & 74.1 & 42.7 & 48.8 & 24.9 & 71.6 & 52.3 \\ 
        ~~$\bullet~$Tent & 25.5 & 29.9 & 39.8 & 61.1 & 28.5 & 75.6 & 26.4 & 2.1 & 30.4 & 50.6 & 59.4 & 74.3 & 28.8 & 30.3 & 9.3 & 68.9 & 40.1 \\ 
        ~~$\bullet~$PL & 28.7 & 27.0 & 30.4 & 56.0 & 45.8 & 75.4 & 37.3 & 5.3 & 28.4 & 56.5 & 60.6 & 73.7 & 14.6 & 29.3 & 9.6 & 71.2 & 40.6 \\
        ~~$\bullet~$SHOT & 24.9 & 28.9 & 38.4 & 57.6 & 29.0 & 73.8 & 19.9 & 2.2 & 32.9 & 47.5 & 60.6 & 74.2 & 24.8 & 26.2 & 7.7 & 69.3 & 38.6 \\
        ~~$\bullet~$EATA & 45.2 & 45.2 & 44.4 & 60.4 & 40.8 & 75.0 & 45.0 & 3.0 & 43.8 & 60.4 & 64.9 & \underline{74.4} & 32.2 & 43.3 & 18.0 & 70.9 & 47.9  \\
        ~~$\bullet~$SAR & 51.0 & 52.5 & 50.4 & 64.7 & 51.6 & 75.7 & 48.1 & 11.6 & 43.3 & \underline{61.1} & 65.0 & 74.1 & 42.7 & 48.8 & 24.9 & 71.6 & 52.3  \\
        ~~$\bullet~$READ & 43.4 & 48.3 & 49.0 & 59.6 & 49.8 & 74.3 & 41.3 & 9.5 & 41.9 & 59.5 & 62.7 & 73.5 & 20.6 & 38.3 & 12.5 & 69.1 & 47.1  \\
        ~~$\bullet~$COME & 27.6 & 25.4 & 29.5 & 55.3 & 43.6 & 74.3 & 35.7 & 2.3 & 28.1 & 55.4 & 59.6 & 72.4 & 12.5 & 27.8 & 8.6 & 70.1 & 39.3 \\
        ~~$\bullet~$TSA & 46.9 & 44.9 & 43.6 & 62.0 & 42.6 & 75.2 & 42.1 & 10.5 & 42.1 & 59.8 & 62.9 & 73.9 & 42.9 & 50.0 & 22.5 & 70.0 & 49.5 \\
        \rowcolor{pink!30}~~$\bullet~$TCR & \underline{56.7} & \underline{58.9} & \underline{55.7} & \textbf{68.9} & \underline{61.2} & \underline{76.2} & \underline{62.2} & \underline{38.8} & \underline{62.8} & \textbf{67.1} & \textbf{75.2} & \textbf{76.6} & \underline{62.3} & \underline{69.6} & \underline{46.9} & \underline{72.0} & \underline{63.2} \\
        \rowcolor{pink!30}~~$\bullet~$REST & \textbf{59.3} & \textbf{60.0} & \textbf{59.9} & \underline{68.0} & \textbf{62.2} & \textbf{76.5} & \textbf{63.8} & \textbf{39.9} & \textbf{63.9} & \textbf{67.1} & \underline{75.0} & \underline{76.0} & \textbf{64.4} & \textbf{70.2} & \textbf{49.1} & \textbf{72.5} & \textbf{64.2} \\
    \cmidrule{1-18}
	\end{tabular}
	}
	 \end{threeparttable}
	 \end{center}
\end{table*}

\begin{table*}[t]
    \vspace{-0.1in}
    \caption{Comparisons with state-of-the-art methods on COCO-O benchmark under \textbf{\textsc{OQS on the text modality}} regarding the Recall@1 metric. The best results are marked in \textbf{bold} and the second best results are \underline{underlined}.}
    \label{tab: coco-o-text-single}
\newcommand{\tabincell}[2]{\begin{tabular}{@{}#1@{}}#2\end{tabular}}
 \begin{center}
 \begin{threeparttable}
    \Large
    \resizebox{0.9\linewidth}{!}{
 	\begin{tabular}{l|ccccc|ccccc|ccccc|>{\columncolor{blue!8}}c}
 	\multicolumn{1}{c}{} & \multicolumn{5}{c}{Character-level} & \multicolumn{5}{c}{Word-level} & \multicolumn{5}{c}{Sentence-level}  \\
 	Methods & OCR & CI & CR & CS & CD & SR & RI & RS & RD & IP & Formal & Casual & Passive & Active & Backtrans & Avg.  \\
    \cmidrule{1-17}
        BLIP ViT-B/16 & \underline{36.6} & 13.7 & 11.6 & \underline{21.5} & 14.3 & 47.4 & 56.0 & \underline{55.0} & 55.9 & 62.6 & 62.0 & 61.9 & 59.9 & 62.6 & 59.1 & \underline{45.3} \\
        ~~$\bullet~$Tent & 24.2 & 3.3 & 2.7 & 8.8 & 4.6 & 46.7 & 54.8 & 52.8 & 55.7 & 62.6 & \underline{62.4} & 62.1 & 59.7 & 62.8 & \underline{59.5} & 41.5 \\
        ~~$\bullet~$PL & 13.5 & 3.3 & 2.9 & 7.8 & 4.7 & 26.9 & 27.6 & 31.6 & 39.5 & 46.2 & 61.8 & 61.6 & 59.0 & 62.5 & 59.0 & 33.9 \\
        ~~$\bullet~$SHOT & 17.7 & 3.2 & 2.8 & 10.9 & 4.3 & 47.1 & 54.6 & 52.5 & 55.6 & 62.1 & \underline{62.4} & 62.1 & 59.8 & 62.8 & \underline{59.5} & 41.2 \\
        ~~$\bullet~$EATA & 34.3 & 6.4 & 5.0 & 17.5 & 7.5 & 47.6 & 55.7 & 54.4 & \underline{56.6} & 62.7 & \underline{62.4} & 62.2 & \underline{60.1} & 62.9 & 59.4 & 43.7 \\
        ~~$\bullet~$SAR & 33.5 & 6.1 & 4.6 & 17.0 & 7.3 & 47.2 & 55.3 & 54.1 & 56.3 & 62.3 & 62.3 & 61.8 & 59.6 & 62.1 & 58.7 & 43.2 \\
        ~~$\bullet~$READ & 20.4 & 4.6 & 3.7 & 12.1 & 7.1 & 37.2 & 39.5 & 36.8 & 47.6 & 60.4 & 61.5 & 61.3 & 59.2 & 62.0 & 58.2 & 38.1 \\
        ~~$\bullet~$COME & 17.9 & 3.4 & 3.0 & 11.2 & 4.5 & 45.8 & 54.3 & 51.8 & 55.4 & 61.5 & 62.3 & \underline{62.3} & 59.5 & \underline{63.0} & 59.3 & 41.0 \\
        ~~$\bullet~$TSA & 34.3 & 12.2 & 10.5 & 19.8 & 12.8 & \underline{47.9} & \underline{56.4} & 54.7 & 56.3 & \underline{62.8} & 62.2 & 62.1 & 60.0 & 62.8 & 59.3 & 44.9 \\
        \rowcolor{pink!30}~~$\bullet~$TCR & 36.2 & \underline{14.7} & \underline{13.1} & 20.8 & \underline{14.4} & 44.4 & 49.8 & 52.0 & 55.7 & 61.8 & 62.3 & 62.2 & 59.5 & 62.8 & 59.3 & 44.6 \\
        \rowcolor{pink!30}~~$\bullet~$REST & \textbf{38.5} & \textbf{15.7} & \textbf{13.9} & \textbf{22.3} & \textbf{15.2} & \textbf{49.9} & \textbf{58.7} & \textbf{56.1} & \textbf{57.0} & \textbf{63.1} & \textbf{63.1} & \textbf{62.8} & \textbf{61.2} & \textbf{63.3} & \textbf{60.5} & \textbf{46.7} \\
    \cmidrule{1-17}
	\end{tabular}
	}
	 \end{threeparttable}
	 \end{center}
\end{table*}

\subsection{More Experiments on the Single-Stream Model}
For the dual-stream retrieval paradigm explored in the manuscript, multi-modal encoders first map both queries and candidates into the common space, and the retrieval results are then obtained by computing the cosine similarity between their features in the common space.
Beyond the dual-stream retrieval paradigm, recent models have further improved performance by employing single-stream networks on top of the initial retrieval results obtained by the dual-stream retrieval paradigm, with BLIP~\cite{BLIP} serving as the most representative method.
To be specific, given the initial retrieval results, the retrieval model first selects the query-candidate pairs with the Top-$k$ highest similarities, and then the single-stream networks would output a binary prediction for each selected pair, where one dimension corresponds to the matching score and the other to the mismatching score.
After that, the matching score is incorporated as an additional reward into the initial retrieval results, yielding the final retrieval results.

To further validate the effectiveness of REST, we conduct additional experiments on the CMR model BLIP, which might be one of the first attempts to explore TTA in the single-stream CMR paradigm.
Since most TTA methods are not specifically designed for the binary predictions produced by single-stream networks, we employ them for updating the parameters of dual-stream networks in the experiments.
It is worth noting that the proposed robust hard-mining loss (Eq.~14) could be directly adopted for the binary predictions and thus facilitates the updates of single-stream networks, which further demonstrates the generality of REST across both single-stream and dual-stream retrieval paradigms.
From the results in Table~\ref{tab: coco-o-image-single}–\ref{tab: coco-o-text-single}, most existing TTA methods fail to deliver performance gains, since the initial retrieval results are not necessarily correct and thus lead to sub-optimal overall retrieval performance.
In contrast, the proposed REST boosts the overall retrieval performance by refining both the initial scores from the dual-stream network and the matching scores from the single-stream network.

\subsection{Efficiency Comparisons}
In this section, we present additional experiments to analyze the efficiency of REST. 
To this end, we perform zero-shot retrieval under the ``GeneralCOCO'' and ``Fashion-Gen'' settings using the pre-trained BLIP and then measure GPU time during the adaptation process.
The results under the ``GeneralCOCO'' setting demonstrate that REST achieves adaptation more efficiently than TSA, which relies on additional parameterized modules. 
Moreover, the results under the ``Fashion-Gen'' setting indicate that REST incurs a negligible extra time cost compared with vanilla Tent and optimization-accelerated EATA (which perform TTA on low-entropy samples), primarily due to the gradient decoupling module.

\begin{table}[t]
\centering
\caption{Efficiency comparisons among different approaches. The unit is seconds.}
\label{tab:time_efficiency}
\Large
\resizebox{0.7\linewidth}{!}{
\begin{tabular}{lccccc}
\toprule
\multirow{2}{*}{Methods} & \multicolumn{2}{c}{GeneralCOCO} & \multicolumn{2}{c}{Fashion-Gen} &  \\
 & TR & IR & TR & IR & Avg. \\
\midrule
Tent & 102 & 175 & 372 & 215 & 216 \\
EATA & 94  & 165 & 355 & 199 & 203 \\
READ & 101 & 178 & 388 & 235 & 226 \\
TSA  & 125 & 193 & 433 & 260 & 253 \\
Ours & 101 & 173 & 384 & 229 & 222 \\
\bottomrule
\end{tabular}
}
\vspace{-0.1in}
\end{table}

\begin{figure}[t]
\centering
\includegraphics[width=0.9\linewidth]{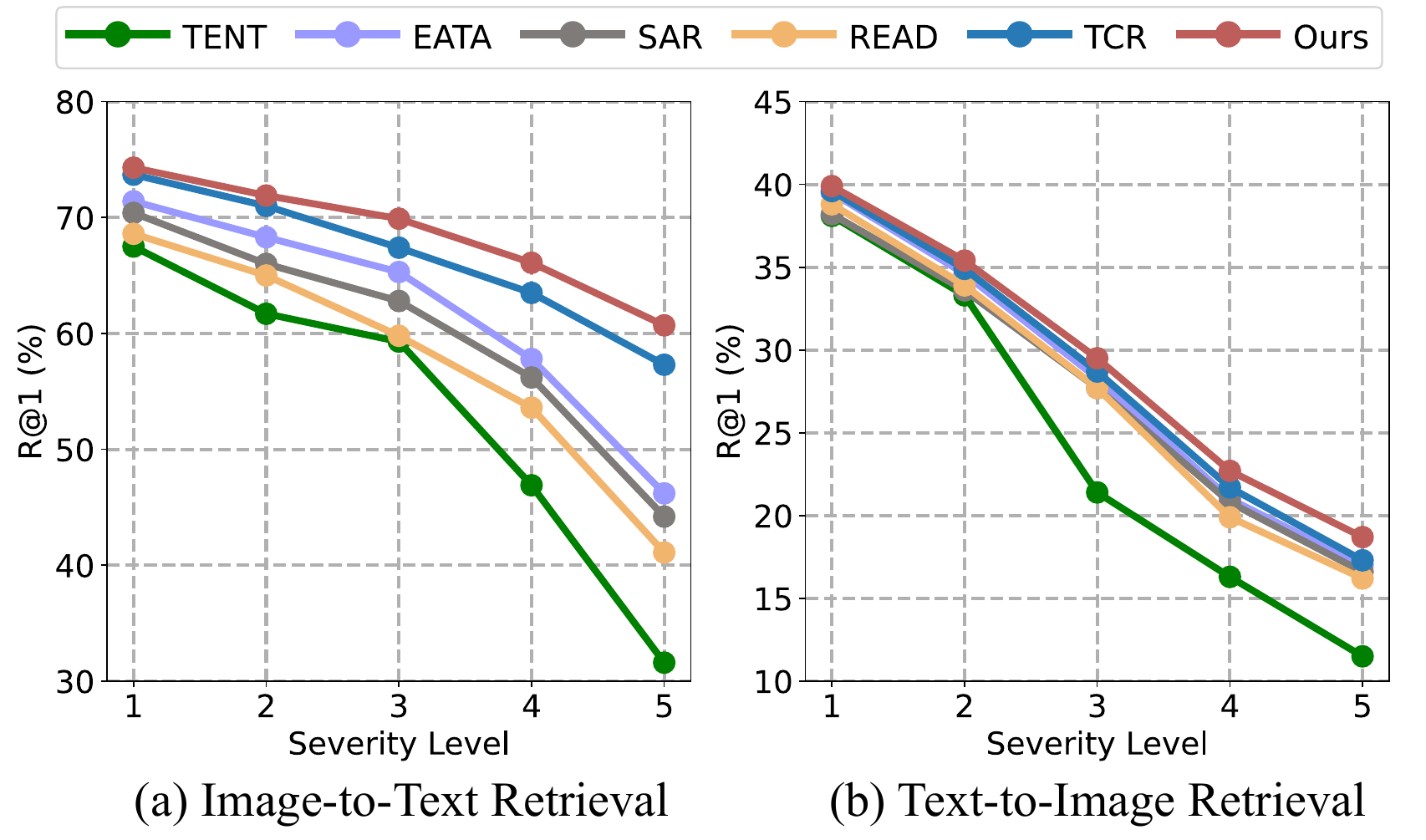}
\caption{
Comparisons on the ``COCO-O'' benchmark with various severities of corruptions.
}
\label{fig: severity}
\end{figure}

\subsection{Comparisons under Various Corruption Severities}
To further validate the effectiveness of the proposed REST, we conduct experiments on the ``COCO-O'' benchmarks under various corruption severities.
Specifically, we introduce 16 types of image corruptions and 5 types of character-level text corruptions in the experiments, each corruption with five levels of severity.
The results in Fig.~\ref{fig: severity} demonstrate that our method not only achieves superior robustness across different severities, but also exhibits significantly slower performance degradation as the corruptions become more severe.
Note that the experiments in the manuscript and Tables~\ref{tab: coco-o-image-large}-\ref{tab: coco-o-text-large} are conducted under corruptions at the maximum severity level.

\end{document}